\documentclass[pmlr]{jmlr}

\providecommand{\main}{.}  
\usepackage{\main/paperpkg}
\usepackage{subfiles}

\title[Fast Statistical Leverage Score Approximation in KRR]{
Fast Statistical Leverage Score Approximation \titlebreak 
in Kernel Ridge Regression}

\author{\Name{Yifan Chen} \Email{yifanc10@illinois.edu} \\
\addr University of Illinois at Urbana-Champaign
\AND
\Name{Yun Yang} \Email{yy84@illinois.edu} \\
\addr University of Illinois at Urbana-Champaign}

\begin{document}
\maketitle

\begin{abstract}
Nystr\"{o}m approximation is a fast randomized method that rapidly solves kernel ridge regression (KRR) problems through sub-sampling the $n$-by-$n$ empirical kernel matrix appearing in the objective function. 
However, the performance of such a sub-sampling method heavily relies on correctly estimating the statistical leverage scores for forming the sampling distribution, which can be as costly as solving the original KRR. 
In this work, we propose a linear time (modulo poly-log terms) algorithm to accurately approximate the statistical leverage scores in the stationary-kernel-based KRR with theoretical guarantees. 
Particularly, by analyzing the first-order condition of the KRR objective, we derive an analytic formula, which depends on both the input distribution and the spectral density of stationary kernels, for capturing the non-uniformity of the statistical leverage scores. 
Numerical experiments demonstrate that with the same prediction accuracy our method is orders of magnitude more efficient than existing methods in selecting the representative sub-samples in the Nystr\"{o}m approximation.
\end{abstract}

\section{Introduction}

The major computational bottleneck of kernel-based machine learning methods, such as kernel ridge regression (KRR) \citep{shawe2004kernel, hastie2005elements}, 
lies in the calculation of certain matrix inverse involving an $n$-by-$n$ symmetric and positive semidefinite (PSD) empirical kernel matrix $K_n \in \mb R^{n \times n}$ over $n$ inputs in the $d$-dimensional Euclidean space $\mb R^d$. 
For most common kernels, the empirical kernel matrix $K_n$ is nearly singular with its effective rank being captured by the so-called {\em effective statistical dimension} \citep{alaoui2015fast, yang2017randomized} $d_{stat}$ that is problem-dependent and can be substantially smaller than the sample size $n$. 
For example, under the optimal choice of the regularization parameter, if the kernel function is a \matern kernel with smoothness parameter $\nu > 0$, then the statistical dimension in the KRR is $d_{stat} = \m O(n^{\frac{d}{2\nu + 2d}})$ \citep{kanagawa2018gaussian}.

\subsection{Related Works}\label{Sec:related_work}

Due to this intrinsic low-rankness of $K_n$, several existing papers developed randomized algorithms, such as the \nystrom method \citep{alaoui2015fast}, randomized sketches \citep{yang2017randomized, ahle2020oblivious}, random Fourier features \citep{rahimi2008random, avron2017random}, and its improvement quadrature Fourier features \citep{mutny2018efficient},
for obtaining a rank $\wt{\m O}(d_{stat})$ ($\wt {\m O}(\cdot)$ means $\m O(\cdot)$ modulo poly-log terms) approximation of $K_n$. 
In particular, sampling-based algorithms, such as the \nystrom method, avoid explicitly constructing the $n$-by-$n$ matrix $K_n$, and only require $\wt{\m O}(nd_{stat})$ evaluations of the kernel function. 
This property is particularly appealing since both the time and space complexity can be reduced to even below the $\m O(n^2)$ benchmark complexity of constructing and storing the empirical kernel matrix. 
From an algorithmic perspective, {\em statistical leverage scores}, as a measure of the structural non-uniformity of the inputs in forming $K_n$, 
can be used for constructing an importance sampling distribution that leads to high-quality low-rank approximations in the \nystrom method. 
We refer the readers to some recent papers \citep{mahoney2011randomized, drineas2012fast} for more details. 

However, the exact computation of the statistical leverage scores bears the same $\m O(n^3)$ time complexity and $\m O(n^2)$ space complexity \citep{mahoney2011randomized} as inverting the $n$-by-$n$ empirical kernel matrix.
Researchers thus turn to the question of whether there is an efficient and accurate method of approximately computing the leverage scores. 
For example, some works \citep{alaoui2015fast, rudi2015less} borrowed the random projection idea \citep{drineas2012fast} in designing an approximation algorithm for computing the statistical leverage scores in the \nystrom method in the context of KRR.
Their algorithm has a worst case time complexity $\m O\big(\frac{n^3}{d^2_{stat}}\big)$ that may exceed the $\m O(nd^2_{stat})$ complexity in subsequent steps for small $d_{stat}$. 
As a refinement, \citet{musco2017recursive} developed Recursive-RLS, a recursive version of the prior algorithm \citep{alaoui2015fast} with overall time complexity $O(n d_{stat}^2)$ by alternating between updating the statistical leverage scores and drawing new subsamples based on the current scores.
SQUEAK \citep{calandriello2017distributed} adapts the algorithm to an online setting, attaining the same accuracy and having the same complexity order with only one pass over the data.
BLESS \citep{rudi2018fast} adopts a path-following algorithm that further reduces the subsampling time complexity to $\m O(\min (\frac{1}{\lambda}, n) d_{stat}^2 \log^2 \frac{1}{\lambda})$ where $\lambda$ is the regularization parameter in the KRR. 
With the choice of $\lambda = \m O(\frac{d_{stat}}{n})$ that leads to the optimal error rate, the complexity of BLESS would be $\wt {\m O} (n d_{stat})$.

\subsection{Our Contribution}
\label{Sec:contribution}

Most previous algorithms are algebraic methods by approximating matrix operations and apply to any positive semidefinite (PSD) kernel.
In this work, we focus on stationary kernels and follow a completely different route of utilizing large sample properties of KRR to develop a new analytical approach for approximating the statistical leverage scores.
Under a classical nonparametric setting, the new approach requires $\wt{\m O}(n)$ time and space complexity, and provably also attains the optimal statistical accuracy in the KRR. 
In a nutshell, rather than approximating the leverage scores by pre-constructing a low-rank approximation to $K_n$ \citep{drineas2012fast}, 
our method uses structural information contained in the kernel function and the underlying input distribution to infer how other inputs influence the statistical leverage score at a given location.
In particular, we derive an explicit and computable formula,
\begin{align*}
\int_{\mb R^d} \frac{1}{p(x_i) +  \lambda/m(s)}\,\dd s,\quad \forall i \in [n],
\end{align*}
where $m(\cdot)$ is the spectral density function of the stationary kernel we use, and $p(x_i)$ is the density of the input $x_i$.
This formula is applied to approximate the rescaled statistical leverage score, which is proportional to the true statistical leverage score, at each observed point.
We also provide the theoretical guarantees that the approximation formula has a vanishing relative error as $n \to \infty$. 

Our development is based on the existing works \citep{silverman1984spline, yang2017frequentist} on the equivalent kernel representation of the KRR solution. 
The consequent theory sheds some light on the behaviour of leverage scores, 
and a simple application is the following rule of thumb: for the \matern kernel with smoothness $\nu$, the statistical leverage score at point $x$ in $\mb R^d$ is proportional to $\min \{1,\,(\lambda/p(x))^{1-d/(2\nu+d)}\}$, 
where the regularization parameter $\lambda = \Theta(d_{stat}/n)$.
This scaling indeed matches the previous research on the asymptotic equivalent of the regularized Christoffel function \citep{pauwels2018relating}, which has intrinsic connections with statistical leverage scores.
We also show through numerical experiments that our method exhibits encouraging performance compared to other methods, which further improves the overall runtime of KRR.

\section{Background and Problem Formulation}

In this section, we set up the notation and briefly introduce the background. 
We begin with a short review on reproducing kernel Hilbert space (RKHS) and kernel ridge regression (KRR). 
After that, we invoke a useful and important formula that represents the norm of any RKHS induced by a stationary kernel via Fourier transforms and Parseval's identity. 
Then we describe the class of \nystrom-method-based approaches as our primary focus for approximately computing the KRR. 
Finally, we introduce the notation of equivalent kernel that plays important roles in both understanding the theoretical properties of the KRR and motivating our proposed method.

\subsection{Reproducing kernel Hilbert space and kernel ridge regression} \label{Sec:rkhs_krr}

\textbf{Reproducing kernel Hilbert space.} 
Particularly, any RKHS is generated by a PSD kernel function $K:\,\m X\times\m X\to \mb R$, 
and there exists a correspondence between any RKHS (or its induced norm $\|\cdot\|_{\mb H}$) and its reproducing kernel (see the books \citep{berlinet2011reproducing,wahba1990spline,gu2013smoothing} for more details). 
Most widely used kernels are stationary, meaning that $K(x,y)$ depends on $x$ and $y$ only through their difference $(x-y)$.
Due to this definition, we may abuse the notation $K(u)$ to mean $K(x,\,x+u)$ for any $x$. 
We will make this stationary kernel assumption throughout the paper.
More specifically, we primarily focus on \matern kernels, although the development can be straightforwardly extended to other stationary ones. 

\textbf{Kernel ridge regression.} 
Consider a dataset $\mathcal{D}_n = \{(x_i, y_i)\}_{i=1}^n$ consisting of $n$ pairs of points in $\mathcal{X} \times \mathcal{Y}$, where $\mathcal{X} \subseteq \mb R^d$ is the input (predictor) space and $\mathcal{Y} \subseteq \mathbb{R}$ the response space. 
To characterize the dependence of the response on the predictor, we assume the following standard nonparametric regression model as the underlying data generating model, $y_i = f^\ast(x_i) + \varepsilon_i,\, i=1,2,\ldots,n,$
where $f^\ast:\,\m X\to \m Y$ is the unknown regression function to be estimated, and the random noises $\{\varepsilon_i\}_{i=1}^n$ are i.i.d.~$\mathcal{N}(0, \sigma^2)$, or any sub-Gaussian distribution with mean zero and variance $\sigma^2$. 
Under the common regularity assumption that the true regression function $f^\ast$ belongs to an RKHS $\mb H$, 
it is natural to estimate $f^\ast$ by an estimator $\wht f$, which minimizes the sum of a least-squares goodness-of-fit term and a penalty term involving the squared norm $\|\cdot\|_{\mb H}$ associated with $\mb H$. 
This leads to the following estimating procedure known as {\em kernel ridge regression} (KRR) \citep{shawe2004kernel, hastie2005elements},
\begin{align} \label{Eqn:loss}
\wht f=\argmin_{f\in \mb H} \Big\{\frac{1}{n} \sum_{i=1}^n (y_i - f(x_i))^2 + \lambda \|f\|_\mb{H}^2\Big\}.
\end{align}
The optimization problem~\eqref{Eqn:loss} appears to be infinite-dimensional over a function space $\mb H$, while indeed its solution can be obtained by solving an $n$-dimensional quadratic program thanks to the representer theorem \citep{kimeldorf1971some}. More precisely, for any two $\m X$-valued vectors $a=(a_1,\ldots,a_p)^T\in\m X^p$ and $b=(b_1,\ldots,b_q)^T\in\m X^q$, 
we use the notation $K(a,b)$ to denote the $p$-by-$q$ matrix whose $(i,j)$-th component is $K(a_i,b_j)$ for $i\in[p]$ and $j \in [q]$. 
Let $X_n=(x_1,\ldots,x_n)^T\in\m X^n$ and $Y_n=(y_1,\ldots,y_n)^T\in\m Y^n$. 
Under this notation, the solution $\wht f$ takes the form as
\begin{equation}\label{eqn:weight_function}
\begin{aligned}
\wht f(x) &\defeq K(x,\,X_n)\, (K_n+n\lambda I_n)^{-1}\, Y_n \\
&= \frac{1}{n}\,\sum_{i=1}^n G_\lambda(x,\,x_i)\, y_i,
\end{aligned}
\end{equation}
where $K_n \defeq K(X_n,X_n)$ is the $n$-by-$n$ empirical kernel matrix.
Here, the weight function $G_{\lambda}:\, \m X\times \m X\to\mb R$ characterizes the impact of each observed pair $(x_i,y_i)$ on $\wht f(x)$, 
and plays an important role in determining the optimal importance sampling weights in the \nystrom method described. 
One key observation is that the weight function $G_\lambda$ depends on the dataset $\m D_n$ only through the design points $\{x_i\}_{i=1}^n$ and the regularization parameter $\lambda$ (which usually depends on the sample size $n$). 
This fact leads to the development of equivalent kernel approximation, as we will come back shortly in Section~\ref{Sec:equiv_k}.

A final remark of the subsection is that solving for $\wht \omega$ requires time complexity $\m O(n^3)$ of inverting an $n$-by-$n$ matrix, which becomes formidable when $n$ gets large. 
The practical demand of computationally scalable methods for implementing the KRR results in a large volume of literature on approximation algorithms, including our current work. 

\subsection{The relation between RKHS norm and Fourier transform}

In this subsection, we describe a useful representation theorem that characterizes the RKHS of a stationary kernel function via the Fourier transform. 
Before formally introducing this theorem, we set up some notation.
We use $L_p(\mb R^d)=\big\{f:\, \mb R^d\to\mb R,\, \int_{\mb R^d} \big|f(x)\big|^p\,\dd x < \infty \big\}$ to denote the space of all $L_p$ integrable functions over $\mb R^d$ for $p \geq 1$.
For any function $f \in L_1(\mb R^d)$, we use $\ms F[f]$ to denote its Fourier transform defined by $\ms F[f](s) = \int_{\mb R^d} f(x)\,e^{-2 \pi \sqrt{-1} x^T s}\,\dd x,\, \mbox{for all $s \in \mb R^d$}$,
and $\ms F^{-1}[g]$ the inverse transform of a function $g$ in the frequency domain as $\ms F^{-1}[g](x) = \int_{\mb R^d} g(s)\, e^{2 \pi \sqrt{-1} x^T s}\,\dd s,\, \mbox{for all $x \in \mb R^d$}$.
We use $\bar{z}$ to denote the complex conjugate of any $z \in \mb C$, the space of complex numbers.
The classical Bochner's theorem shows that the Fourier transform of any PSD and stationary kernel function is nonnegative. 

The following theorem, which is not new but less well-known in the machine learning literature, provides a characterization of the RKHS with kernel $K$ through its spectral density function.
Similar statements can be found in two previous papers (\citealt[Thm~10.12]{wendland2004scattered} and \citealt[Appendix A]{belkin2018approximation}).
For the sake of completeness, we also include a proof in Section~\ref{sec:thm2} in the appendix, which is motivated by \citet{fukumizu2008elements}. 
\begin{theorem}[Fourier representation of RKHS]\label{Eqn:RKHS_Fourier}
	Let function $m$ be the spectral density of a PSD and stationary kernel $K$, and $\mb H$ the associated RKHS. For any $f,g\in \mb H$, we have
	\begin{equation} \label{Eqn:freq_domain}
	\begin{aligned}
	\|f\|_{\mb H}^2 &= \int_{\mb R^d} \frac{\big|\ms F[f](s)\big|^2}{m(s)}\, \dd s, \\
    \quad \textrm{and} \quad
	\langle f,\, g\rangle_{\mb H} &= \int_{\mb R^d} \frac{\ms F[f](s)\cdot \overline{\ms F[g](s)}}{m(s)}\, \dd s.
	\end{aligned}
	\end{equation}
	In particular, the RKHS can be represented by $\mb H=\big\{f:\, \|f\|_{\mb H}^2=\int_{\mb R^d} \big|\ms F[f](s)\big|^2/m(s)\, \dd s<\infty\big\}$.
\end{theorem}

\subsection{\nystrom methods with importance subsampling}

From expression~\eqref{eqn:weight_function} of the KRR solution $\wht f$, the $\m O(n^3)$ computational bottleneck comes from the inversion of the $n$-by-$n$ matrix $(K_n + n \lambda I)$. 
When design points $\{x_i\}_{i=1}^n$ are distinct and sample size $n$ is large, the empirical kernel matrix $K_n$ often has a high condition number and is nearly low-rank. 
In particular, several recent studies \citep{alaoui2015fast, yang2017randomized} show that both the computational and the statistical hardness of the KRR are captured by a quantity called the {\em statistical dimension} defined as
\begin{equation} 
\begin{aligned}
\label{eqn:stat_dim}
d_{stat} &\defeq {\rm Tr}\big(K_n(K_n+n\lambda I_n)^{-1}\big) \\
&= \frac{1}{n}\sum_{i=1}^n G_\lambda(x_i,\,x_i),
\end{aligned}
\end{equation} 
where ${\rm Tr}(A)$ means the trace of a matrix $A$. 

The statistical dimension $d_{stat}$ approximately counts the number of eigenvalues of the rescaled kernel matrix $n^{-1}K_n$ whose values are above the threshold $\lambda$. 
Therefore, computing $\wht f$ roughly amounts to solving an $d_{stat}$-dim quadratic program,
and for most stationary kernels $d_{stat}$ would be much smaller than $n$.
For example, for a \matern kernel with smoothness parameter $\nu$ being a positive half integer, 
$d_{stat} = \m O(n^{\frac{d}{2\nu+2d}})$ \citep{tuo2020improved}.
Due to the intrinsic low-rankness of $K_n$, the so-called \nystrom method, which replaces $K_n$ with its low-rank approximation $L_n$, has been used for approximately solving the KRR \citep{gittens2016revisiting, kumar2009sampling, williams2001using}. 
Specially, the \nystrom approximation of $K_n$ is the matrix $L_n = K_n S(S^T K_n S)^\dagger S^T K_n$,
where $A^\dagger$ denotes the Moore-Penrose pseudoinverse of a matrix $A$, and $S \in \mb R^{n \times d_{sub}}$ is a zero-one subsampling matrix whose columns are a subset of the columns in $I_n$, indicating which $d_{sub}$ observations have been selected. 
We use $\wht f_{L_n}$ to denote the approximate KRR solution obtained by replacing $K_n$ with $L_n$ in expression~\eqref{eqn:weight_function}. 

Following the paper~\citep{alaoui2015fast}, we will refer to the diagonal elements of matrix $K_n(K_n+n\lambda I_n)^{-1}$ as the {\em statistical leverage scores} $\{\ell_i\}_{i=1}^n$ associated with the kernel matrix $K_n$. 
It is straightforward to verify from identity~\eqref{eqn:weight_function} that the $i$-th diagonal element of $K_n(K_n+n\lambda I_n)^{-1}$ is precisely $\frac{1}{n} G_\lambda(x_i,x_i)$,
and thus we call $G_\lambda(x_i,x_i)$ the rescaled statistical leverage score.

The next result~\citep[Theorem~3]{alaoui2015fast} (after adapting to our notation) shows that if we use a randomized construction of $S$ by sampling $d_{sub} = \m O(d_{stat} \log(n))$ columns from $I_n$ with a proper distribution $\{q_i\}_{i=1}^n$ over $[n]$ approximately proportional to the statistical leverage scores, 
the resulting approximate solution $\wht f_{L_n}$ would attain the same statistical in-sample risk (up to a constant) as the original KRR solution $\wht f\,$. 
Here, we define the in-sample prediction risk for any regression function $f$ as $R_n(f) = \|f - f^\ast\|_n^2:\,=n^{-1}\,\sum_{i=1}^n\big(f(x_i)-f^\ast(x_i)\big)^2$.
\begin{theorem}[\nystrom approximation accuracy]\label{thm:nystrom_error}
	Fix $\rho\in(0,1/2)$. 
	Let $L_n$ be the \nystrom approximation of $K_n$ with $S$ being formed by choosing $d$ columns randomly with replacement from the columns of the identity matrix $I_n$ according to a probability distribution $\{q_i\}_{i=1}^n$. 
	Suppose there exists some $\beta\in(0,1]$ such that $q_i\geq \beta\, G_{\lambda}(x_i,\,x_i)/\sum_{i=1}^nG_\lambda(x_i,\, x_i)$, and
	\begin{align*}
	d_{sub} \geq C_1\frac{d_{stat}}{\beta}\, \log\Big(\frac{n}{\rho}\Big) \quad\mbox{and}\quad \lambda\geq \frac{C_2}{\min_i G_\lambda(x_i,\,x_i)},
	\end{align*}
	where $d_{stat}$ is defined in~\eqref{eqn:stat_dim}. Then it holds with probability at least $1-2\rho$ that
	$R_n(\wht f_{L_n})\leq C_3 R_n(\wht f\,)$. Here $C_i$, $i=1,2,3$, are absolute constants.
\end{theorem}
This theorem indicates that the problem of approximately solving the KRR reduces to that of approximately estimating the statistical leverage scores $\{\ell_i\}_{i=1}^n$.
Directly computing these leverage scores using SVD requires inverting an $n$-by-$n$ matrix and is as costly as solving the original KRR optimization~\eqref{eqn:weight_function}.
Finding purely numerical methods for approximating these leverage scores can also be quite challenging. 
For example, the approximation algorithm used in the paper~\citep{alaoui2015fast} has $\m O\Big(\frac{n^3}{d_{stat}^2}\Big)$ time complexity, which significantly exceeds the $\m O(n d_{stat}^2)$ complexity for forming $L_n$ and solving for $\wht f_{L_n}$ in the \nystrom approximation when $d_{stat} \ll \sqrt{n}$ (which holds for any M\'{a}tern kernel).

\subsection{Equivalent Kernel}\label{Sec:equiv_k}

Our method is initially motivated by a notion, \emph{equivalent kernel}, which was first introduced by \citet{silverman1984spline}.
The author showed that in the context of smoothing spline regression, as sample size $n$ goes to infinity the weight function $G_\lambda(\cdot,\,\cdot)$ after a proper rescaling approaches a limiting kernel function, called the equivalent kernel.
A recent work~\citep[Theorem~2.1]{yang2017frequentist} extends the context from smoothing spline regression to general kernel ridge regression. 
They proved that for a general kernel $K$, under the stochastic assumption that design points $\{x_i\}_{i=1}^n$ are i.i.d.~distributed according to a common distribution over $\m X$, 
there exists some equivalent kernel $\bar K_\lambda:\,\m X\times \m X\to \mb R$, 
such that the KRR estimator is asymptotically the same as a simple kernel type estimator with kernel function $\bar K_\lambda$, that is, under a suitable choice of diminishing regularization parameter $\lambda$, the following approximation error bound holds with probability tending to one as $n\to\infty$,
\begin{align}\label{eqn:KRR_expansion}
\sup_{x\in\m X}\Big|\wht f(x) - \frac{1}{n}\sum_{i=1}^n \bar K_\lambda(x,\,x_i)\, y_i\Big| \leq \gamma_n\, \sqrt{\lambda},
\end{align}
where $\gamma_n \to 0$ as $n\to\infty$ and $\sqrt{\lambda}$ matches the maximal magnitude of the estimation error $\sup_{x\in\m X}|\wht f(x) - f^\ast(x)|$. 
Formula~(\ref{eqn:KRR_expansion}) predicts a theoretical limit of the statistical leverage score $\ell_i$ as $\bar K_\lambda(x_i,\,x_i)$ that is \emph{independent} of the design points other than $x_i$.
In other words, the KRR estimator can be expressed as a linear combination of $y_i$'s, 
whose coefficients only depend on the corresponding design point $x_i$. 
Besides, under mild conditions \citep{yang2017frequentist} on the design distribution, these coefficients admit a limit in probability (that can be characterized via an ``equivalent kernel" function) as $n \to \infty$.
The theoretical result motivates us to seek a computationally efficient method for the leverage scores through approximating these $\bar K_\lambda(x_i,\,x_i)$'s.

\section{Leverage Score Approximation via Spectral Analysis}

We now turn to the main results of this work.
At a high level, we propose our method and give a sketch of the derivation via Fourier transform. 
We also provide the analysis of the corresponding time complexity and consider some stationary kernels to which our method would be applied.
Eventually, we prove that, by taking \matern kernels as an example, our method would attain an optimal prediction risk in the KRR.

\subsection{The proposed algorithm and a heuristic derivation}\label{sec:heuristic}

Under the notation above, we formally propose the explicit formula of our approximation method as
\begin{align}\label{Eqn:lv_app}
\wt K_{\lambda}(x_i,\,x_i) = \int_{\mb R^d} \frac{1}{p(x_i) +  \lambda/m(s)}\,\dd s,
\end{align}
where $p(x_i)$ is the density of $x_i$ defined in Section~\ref{Sec:contribution}.
With Eqn~(\ref{Eqn:lv_app}), we would use Algorithm~\ref{Alg:sa} below to approximate $G_\lambda(x_i,\,x_i)$ for some fixed {point ${x_i \in \mb R^d}$}. 
\begin{algorithm}
\label{Alg:sa}
\SetAlgoLined
\KwIn{the input samples $X_n$ and the spectral density $m(\cdot)$ of the stationary kernel used}
\KwOut{A descrete sampling distribution $\{q_i\}_{i=1}^n$}
Initialize the sampling distribution $q_i = 0, \forall i = 1, \dots, n$\; 
Estimate the density $p_i$ of the samples $x_i, \forall i = 1, \dots, n$\; 
\For{i=1:m}{
    Compute the integration~(\ref{Eqn:lv_app}) with $p_i$, and assign the value to $q_i$ 
}
Denote $Q = \sum_{i=1}^n q_i$\; 
Update $q_i$ as $q_i / Q, \forall i = 1, \dots, n$\;
\caption{Estimation of the leverage scores}
\end{algorithm}

To derive the formula (\ref{Eqn:lv_app}), by setting $y_i=n$, $y_j=0$ (for any $j\neq i$), we transform the objective value in the KRR optimization problem~\eqref{Eqn:loss} to the following functional:
\begin{align*}
A_{n,x_i}(f) &= \frac{1}{2n} \sum_{j=1}^n f(x_j)^2 + \frac{1}{2} \lambda \,\|f\|_\mathcal{H}^2 - f(x_i) \\
&= \frac{1}{2} \int_{\mb R^d} f(x)^2 \,\dd F_n(x) + \frac{1}{2} \lambda \,\|f\|_\mathcal{H}^2 - f(x_i),
\end{align*}
for any function $f \in \mb H$, where $F_n$ denotes the empirical distribution of $\{x_i\}_{i=1}^n$ and the integral is the Riemann–Stieltjes integral.
The minimizer of $A_{n,x_i}(f)$ would simply be $\wht f(\cdot) = \frac{1}{n} \sum_{j=1}^n G_\lambda(\cdot, x_j) y_j = G_\lambda(\cdot, x_i)$,
due to the independence of $G_\lambda$ from $\{y_i\}_{i=1}^n$.
Therefore, it suffices to analyze and understand this functional $A_{n,x_i}$ for any $x_i \in \m X$.
Now we assume that there exists a nice cdf $F$ over $\m X$ that admits a Lipschitz continuous density function denoted by $p$, so that the sup-norm $\tau(n)=\|F_n-F\|_\infty:\,=\sup_{x\in\m X}|F_n(x)-F(x)|$ is small. 
We further remark here that in the most common case, $\{x_i\}_{i=1}^n$ are i.i.d.~from pdf $p$, and then $\tau(n) \leq C\sqrt{\log n/n}$ holds with high probability due to the Glivenko-Cantelli theorem \citep{van1996weak}.
Since $A_{n, x_i}$ is convex, finding its optimum amounts to finding the unique root of its functional derivative (such as the Gateaux derivative), 
which is a linear operator $DA_{n,x_i}(f):\, \mb H \to \mb H$ defined at each $f\in\mb H$ as
\begin{align*}
DA_{n,x_i}(f)(u) =\, &\int_{\mb R^d} f(x)\,u(x) \,\dd F_n(x) \\
 &+ \lambda \,\langle f,\,u\rangle_{\mb H} - u(x_i), \quad\mx{for any} \, u \in \mb H.
\end{align*}
Since $F_n$ can be well-approximated by $F$ under the assumption $\tau(n)\to 0$ and the solution $G_\lambda(\cdot,\, x_i)$ is expected to approach to a Dirac delta function centered at $x_i$ as $n\to \infty$ (which will be formalized in Section~\ref{sec:ek_ppt} in the appendix),
the above derivative can thus be approximated by a simpler population-level functional
\begin{align*}
DA_{x_i}(f)(u)=\, &p(x_i) \int_{\mb R^d} f(x)\,u(x) \, \dd x \\
&+ \lambda \,\langle f,\,u\rangle_{\mb H} - u(x_i),\quad\mx{for any}\, u \in \mb H,
\end{align*}
where we replace the differential $\dd F_n(x)$ with its local approximate $p(x_i) dx$.
This new operator admits a simpler form in the frequency domain thanks to Parseval's theorem (cf. Theorem~\ref{thm:Parseval} and Lemma~\ref{lem:RKHS_norm} in the appendix),
\begin{align*}
DA_{x_i}(f)(u)=&\,\int_{-\infty}^\infty \Big(p(x_i)\, \ms F[f](s) + \frac{\lambda}{m(s)}\,\ms F[f](s) \\
& - \exp\big\{-2 \pi \sqrt{-1} x_i s\big\}\Big)
\,\overline{\ms F[u](s)}\,\dd s.
\end{align*}
Therefore, the unique root of $DA_{x_i}(\cdot)$, denoted by $\wt K_{\lambda}(\cdot,\,x_i)$, can be obtained by equating the function inside the big parenthesis in the preceding display to zero, which is the inverse Fourier transform of
\begin{align}\label{eqn:eqv_kernel}
\frac{\exp\big\{-2 \pi \sqrt{-1} \dotp{x_i}{s} \big\}}{p(x_i) + \lambda/m(s)},\quad s \in \mb R^d,
\end{align}
or $\wt K_{\lambda}(\cdot,\,x_i) = \ms F^{-1}\big[ \big(p(x_i) + \lambda/m(s)\big)^{-1}\big](\,\cdot\, - x_i)$ due to the translation property of the Fourier transform. 
Replacing $x_i$ with the $i$-th design point $x_i$ leads to the following quantity
\begin{align*}
\wt K_{\lambda}(x_i,\,x_i) = \int_{\mb R^d} \frac{1}{p(x_i) +  \lambda/m(s)}\,\dd s,
\end{align*}
due to the inverse Fourier transform formula. 
We show its applications to \matern kernels as follows.




\textbf{Example (\matern kernels)}: \matern family \citep{matern2013spatial} is a class of isotropic kernels widely used in spatial statistics. 
The kernel function is expressed as $C_\nu (x, y) = C_\nu (x-y) = \frac{2^{1 - \nu}}{\Gamma(\nu)} (a \|x-y\|)^\nu B_\nu(a\|x-y\|)$, 
where $B_\nu$ is a modified Bessel function of the second kind, $\nu$ is a smoothness parameter (usually half integers), and $a > 0$ a scale parameter. 
Here we slightly abused the notation since $C_\nu$ is stationary. 
An important fact about the \matern kernel $C_{\nu}$ is that its associated RKHS is the $(\nu+d/2)$-th order Sobolev space (we can verify it by plugging the following Fourier transform $m_\alpha(\cdot)$ into Theorem~\ref{Eqn:RKHS_Fourier}). 
The notation $\alpha=\nu+d/2$ is hence used to denote the underlying smoothness level associated with kernel $K_{\alpha} \defeq C_{\alpha-d/2}$,
and the rescaled leverage approximation $\wt K_\lambda$ associated with $K_\alpha$ satisfies $\wt K_\lambda(\cdot\,,\,x_i)=\ms F^{-1} \Big[\Big(p(x_i) + \lambda D_\alpha^{-1} \big(a^2 + \|s\|^2\big)^{\alpha}\Big)^{-1}\Big](\cdot - x_i)$,
where $a=\sqrt{2\nu}, D_\alpha=\Gamma(\alpha) a^{2 \alpha-1} \pi^{-d/2} / \Gamma(2\alpha-1)$.
For general density function $p$, the integral formula~\eqref{Eqn:lv_app} with $m=m_\alpha$ provides a rule of thumb on how the statistical leverage score depends on the local input density as $\ell_i \propto \min \{1,\, ({\lambda}/{p(x_i)})^{1-d/(2\alpha)}\}$,
which implies a relatively large value over those under-sampled regions with small $p(x_i)$.

\subsection{Computational complexity}

To give the complexity analysis, we first stress that in our theoretical development the dimension $d$ is either fixed or at most slowly (e.g. logarithmically) increases with the sample size $n$.
Beyond this setting, at least theoretically, the smallest subsampling size (via statistical dimension $d_{stat}$, which is $\m O((\log n)^{\frac{d}{2}})$ even under a Gaussian kernel) becomes comparable to $n$, making subsampling meaningless due to the curse of dimensionality.
In addition, classical nonparametric literature \citep{silverman1984spline, yang2017frequentist, van2009adaptive} suggests that a dimension $d$ of order $o(\log n)$ is necessary to make any estimator consistent.

With the requirement on $d$ above, we claim $\wt K_{\lambda}(x_i,\,x_i)$ can be efficiently computed in $\wt{\m O}(n)$ time. 
Specifically, the overall complexity includes two parts, numerical integration, and density estimation.
A key observation here is that for both parts the error rates are only required to be sub-optimal,
and $o(1)$ relative error suffices to guarantee the optimality of the error rate in the KRR.
With such a high tolerance of error, the two parts above could both be implemented in $\wt{\m O}(n)$ time as claimed (see Section~\ref{sec:polar},~\ref{sec:density_estimation} in the appendix for more details).

In particular, the overall complexity can be made at most polynomial in the dimension $d$.
For the integration part, we can avoid the exponential dependence on $d$ by applying a polar coordinate transformation to reduce the multivariate integral~(\ref{Eqn:lv_app}) to a univariate integral (c.f. Section~\ref{sec:polar} in the appendix).
For density estimation, some advanced methods are able to generate $n$ density estimates at sample design points in $\m O(n d \log n)$ time with relative approxiamtion error (difference between accurate KDE and approximation methods) of magnitude $\m O((\log n)^{-1/2})$ (such as ASKIT~\citep[Eqn~(3.3)]{march2015askit}, HBE~\citep[Theorem~12]{charikar2017hashing}, and modified HBE~\citep[Theorem~1]{backurs2019space}).
(Those methods only aim to approximate the original KDE, and hence have no requirements on the density but the kernel used in KDE.
The exact set of assumptions for modified HBE are provided in Section~\ref{Sec:hbe} in the appendix.)
In practice, when the KRR problem of interest is not high dimensional($d=\omega(1)$), 
we are even able to efficiently estimate the density with the optimal error rate in $\m O(n (\log n)^d)$ time, by some classical approaches (such as KD-tree methods~\citep{ivezic2014statistics}, fast multipole methods \citep{greengard1997fast}, and fast Gauss transforms \citep{greengard1991fast}),
which are empirically be even faster than the advanced KDE methods above.

As a closing of this subsection, we leave a comment regarding Gaussian kernels.
It seems Gaussian kernels have a low statistical dimension $d_{stat} = \m O( (\log n)^{d/2})$, 
which may allow previous leverage approximation methods, such as BLESS, to have a time complexity comparable to our method. 
We point out here the complete expression for the scale of $d_{stat}$ should be $\m O(\sigma^{-d} (\log (n \sigma^{2d}))^{d/2})$ \citep{yang2017randomized}, where $\sigma$ is the bandwidth of the Gaussian kernel used.
It implies the statistical dimension of Gaussian Kernels would actually be heavily impacted by the bandwidth $\sigma$.
However, as we hope to attain the optimal error rate in KRR, we need to decrease the bandwidth $\sigma$ of Gaussian kernels to $\m O (n^{-c})$ ($c \in (0, \frac{1}{2d})$) to significantly enrich the associated RKHS \citep{van2009adaptive}. 
As a trade-off, the magnitude of $d_{stat}$ would simultaneously be increased to a polynomial of $n$,
which is comparable to the scale of $d_{stat}$ using a proper \matern kernel.
Therefore, generally Gaussian kernels cannot enable the previous leverage approximation methods to enjoy the $\wt{\m O}(n)$ complexity.

\begin{figure}[h]
\vspace{.3in}
\centerline{\includegraphics[width=\textwidth]{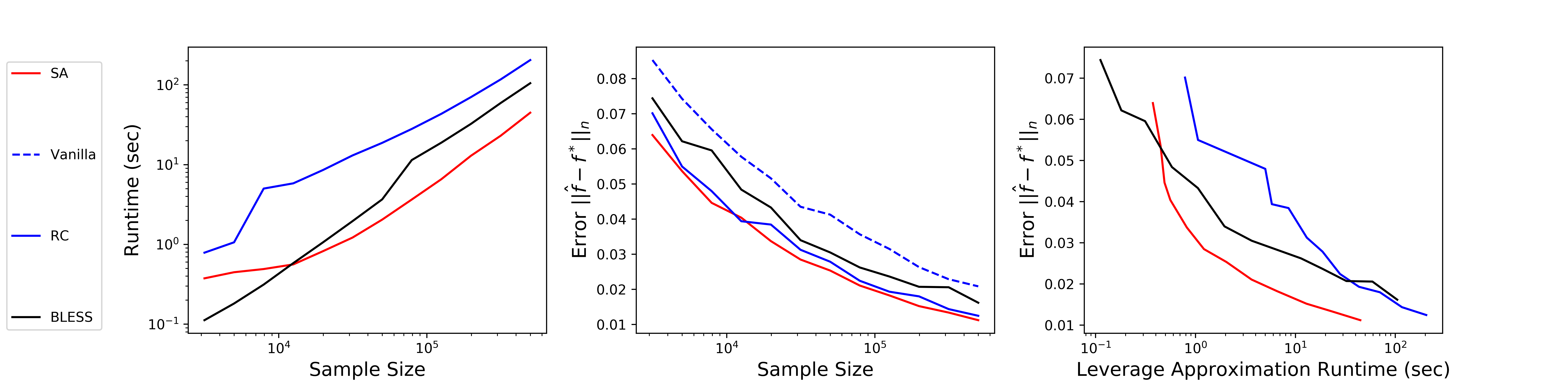}}
\vspace{.3in}
\caption{Run time vs. error tradeoff.}
\label{time_error}
\end{figure}

\subsection{Theoretical results}

We focus on the \matern kernel $K_\alpha$ (proof for other stationary kernels can be developed similarly) and define the effective bandwidth parameter, an important auxiliary parameter analogous to the bandwidth of a Gaussian kernel, as $h \defeq \lambda^{1/(2\alpha)}$,
which indicates the smoothness of the functions in the corresponding RKHS. 
When $\lambda$ is set to obtain the minimax-optimal KRR estimator $\wht f$, the scale of $h$ would therefore be $\Theta(n^{-1/(2\alpha + d)})$.
Our first result provides an explicit error bound on $|G_\lambda(\cdot,\, x_i)-\wt K_\lambda(\cdot,\,x_i)|$ in a neighborhood of $x_i$. 
Particularly, Lemma~\ref{Lem:EqKernel_property_matern} in the appendix shows that the equivalent kernel $\wt K_\lambda(\cdot,x_i)$ resembles a Dirac delta function centered at $x_i$ with radius $\m O(h)$, 
so it suffices to characterize its difference with $G_\lambda(\cdot,\,x_i)$ in the local neighborhood. 
The regularity assumptions of our results are listed as follows:
\begin{assumption}
	There exists a distribution $F$ whose density function $p$ is Lipschitz continuous, so that $\tau(n)=\|F_n-F\|_\infty \leq C_0\, h^2$ for some constant $C_0>0$.
	\label{a.1}
\end{assumption}
\begin{assumption}
	The density function $p$ is strictly positive at the points $x_i, i=1, \dots, n$. Besides, for each $i$, there exists some $\delta(x_i)$, such that $p(x) \geq \frac{p(x_i)}{2}$ for all $x \in B\left(x_i, \delta(x_i)\right) \defeq \{x; \|x-x_i\| < \delta(x_i)\}$, and $\delta(x_i) \geq Ch\log(\frac{1}{h}), \forall i \in [n]$ for some sufficiently large constant $C$ independent of $(x_i, h, n)$.
	\label{a.2}
\end{assumption}
Assumption~\ref{a.1} needs the empirical distribution $F_n$ to be well-approximated by a smooth cdf $F$ (c.f.~Section~\ref{sec:heuristic}). Assumption~\ref{a.2} requires $x_i$ to be a proper interior point of the support of $p$, or at least $\Theta(h\log(\frac{1}{h}))$ far away from the zero density region. 
These two assumptions are mild. 
Assumption~1 is indeed a generalized version of a common assumption “$x_i$’s are drawn iid from a distribution F with Lipschitz density p”, while also compatible with fixed design setting.
Assumption~1 would automatically be satisfied when the quoted assumption holds, with $\tau(n) = \m O((d/n)^{1/2})$ guaranteed by multivariate Glivenko-Cantellli theorem \citep[p.~828][Theorem~1]{shorack2009empirical}.
Another remark about Assumption~\ref{a.1} is that the Lipschitz continuity of the density $p$ is enough for KDE to produce consistent estimation with mean squared error $\m O(n^{-\frac{2}{d+2}})$ \citep{walter1979probability}, 
which is even dominated by the approximation error of the KDE approximation methods mentioned above. 
(Thus now we can conclude the KDE methods above with $\m O(n d (\log n))$ time complexity could provide sufficient accuracy.)
For Assumption~2, as long as $p(x_i) > 0$ and $p$ is continuous, it always holds in the large scale setting as $n \to \infty$ since $h \log(\frac{1}{h}) \to 0$ as $h \to 0$. 
Moreover, we only need to verify it for the observed design points $\{x_i\}_{i=1}^n$ in conjunction with Theorem~\ref{thm:nystrom_error}, 
and by definition $p(\cdot)$ is automatically positive at these observed points.
\begin{theorem}[Leverage score approximation]\label{thm:main_result}
	If Assumptions~\ref{a.1} and~\ref{a.2} hold, then for any $i \in [n]$
	\begin{align*}
	\sup_{x\in B(x_i,\delta(x_i))} \big|G_\lambda(x,\,x_i) &- \wt K_\lambda(x,\,x_i)\big| \\
        &\leq C_{x_i}\, h^{-d}\big(\tau(n)\,h^{-d}+h\big).
	\end{align*} 
	In particular, the relative error of approximating $G_\lambda(x_i,\,x_i)$ by the integral~\eqref{Eqn:lv_app} satisfies
	\begin{align*}
	\frac{\big|G_\lambda(x_i,\,x_i) - \wt K_\lambda(x_i,\,x_i)\big|}{\big|G_\lambda(x_i,x_i)\big|} \leq C'_{x_i}\, \big(\tau(n)\,h^{-d} + h\big).
	\end{align*}
	Here we may choose $C_{x_i} = C\max\{1,p^{-1/2}(x_i)\}$ and $C'_{x_i} = C \max\{1,\,p^{1/2-d/(2\alpha)}(x_i)\} \sqrt{p(x_i)}$ for some constant $C$ independent of $(x_i,h,n)$.
\end{theorem}
The proof of this theorem relies on a novel Sobolev interpolation inequality that bounds the localized sup-norm via the RKHS norm $\|\cdot\|_{\mb H}$ plus a localized $L_2$ norm.
We remark that the rescaled leverage score $G_\lambda(x_i,\,x_i)$ and our approximation $\wt K_\lambda(x_i,\,x_i)$ would both increase to infinity as $n \to \infty$, and the upper bound for the difference between them would also diverge as shown in the first inequality above; 
however, the relative error, the quantity of our interest, would shrink. 
Combining the above with Theorem~\ref{thm:nystrom_error}, we can show our approximation leads to an optimal prediction risk in the approximated KRR.
\begin{theorem}[\nystrom approximation]
	Suppose Assumptions~\ref{a.1},~\ref{a.2} hold for each $x_i, i=1,\ldots,n$. Under the same setting and conditions Theorem~\ref{thm:nystrom_error}, if the importance sampling weights are chosen as $q_i= \wt K_\lambda(x_i,\,x_i)/\sum_{i=1}^n\wt K_\lambda(x_i,\,x_i)$, we have $R_n(\wht f_{L_n}) \leq C_3 R_n(\wht f)$ with probability at least $1 - 2\rho$.
\end{theorem}

\section{Experiments}
In this section, we evaluate our leverage score approximation method on both synthetic and real datasets. 
The algorithms below are implemented in unoptimized Python code, run with one core of a server CPU (Intel Xeon-Gold 6248 @ 2.50GHZ) on Red Hat 4.8.
Specifically, we perform the numerical integration and density estimation as described in Section~\ref{sec:polar} and \ref{sec:density_estimation} in the appendix.
Due to the limited space, the complete settings of the experiments below and more supplementary results can be found in Section~\ref{sec:exp_details} in the appendix.

\subsection{Performance in kernel ridge regression}\label{Sec:pred_risk}
We compare the in-sample prediction error of \nystrom methods in KRR, as well as the corresponding leverage approximation time among all the competing algorithms: 
uniform sampling (hereinafter referred to as ``Vanilla"), Recursive-RLS (RC), \citep{musco2017recursive}, Bottom-up Leverage Scores Sampling (BLESS) \citep{rudi2018fast} and our proposed spectral-analysis-based method (SA). 
(The Monte Carlo approximation for the regularized Christoffel function \citep{pauwels2018relating} in practice reduces to directly computing leverage scores and is thus omitted.)
In the experiment, we generate design points $\{x_i\}_{i=1}^n$ with $n \in [2000, 500000]$ from a 3-D bimodal distribution (see Section~\ref{sec:exp_details} in the appendix for the definition).
We use squared in-sample estimation error $\|\wht f - f^\ast\|_n^2$ as the evaluation metric.
All the results reported in Figure~\ref{time_error} are averaged over 30 replicates.
We remark that in the left or the right subplot there is no curve for ``Vanilla" method, as this method assumes the leverage scores are uniform and thus takes no time to approximate.

In Figure~\ref{time_error}, we can observe ``Vanilla" fails to capture the information of the entire design distribution as expected, 
as with high probability, only few data points from the small mode would be sampled. 
For RC, BLESS, and our method, although they are all able to capture the non-uniformity, 
our method has the best runtime versus error trade-off, especially when $n$ is large.
Particularly, when $n = 5 \times 10^5$ our method takes $35.8$s to approximate the leverage scores, while RC and BLESS respectively take a higher cost—around $94.3$s and $167$s—due to their higher complexities.

\begin{table*}[t]
  \caption{Statistical Leverage Score Approximation Accuracy}
  \label{tab:lvg_approx}
  \centering
\begin{tabular}[]{lllllllllll}
\toprule
 & \multicolumn{3}{l}{RQC}  & \multicolumn{3}{l}{HTRU2} & \multicolumn{3}{l}{CCPP}\\
\cmidrule(r){2-10}
Method & Time & $\bar r$ & $5^{th}/95^{th}$ & Time & $\bar r$ & $5^{th}/95^{th}$ & Time & $\bar r$ & $5^{th}/95^{th}$ \\
\midrule
SA & 0.40 & 1.01 & 0.87/1.13 
& 2.23 & 1.04 & 0.77/1.26 
& 0.48 & 1.00 & 0.79/1.21 \\
Vanilla & - & 1.06 & 0.64/1.40 
& - & 1.13 & 0.53/1.63 
& - & 1.04 & 0.72/1.33 \\
RC & 6.97 & 1.03 & 0.75/1.33 
& 2.15 & 1.05 & 0.75/1.27 
& 9.21 & 1.02 & 0.82/1.24 \\
Bless & 3.83 & 1.03 & 0.74/1.33 
& 1.63 & 1.07 & 0.67/1.32 
& 5.25 & 1.02 & 0.81/1.24 \\
\bottomrule
\end{tabular}
\end{table*}

\subsection{Statistical leverage scores accuracy}

We empirically validate that the approximation $\wt K_\lambda(x_i,x_i)$ approaches the rescaled statistical leverage score $G_\lambda(x_i, x_i)$ as guaranteed by our theory.
In particular, we compare the true rescaled leverage and our approximation for samples from one-dimensional (for the ease of visualization) Unif$[0, 1]$, Beta$(15, 2)$, and a bimodal distribution.

\begin{figure}[h]
\vspace{.3in}
\centerline{\includegraphics[width=3.5in]{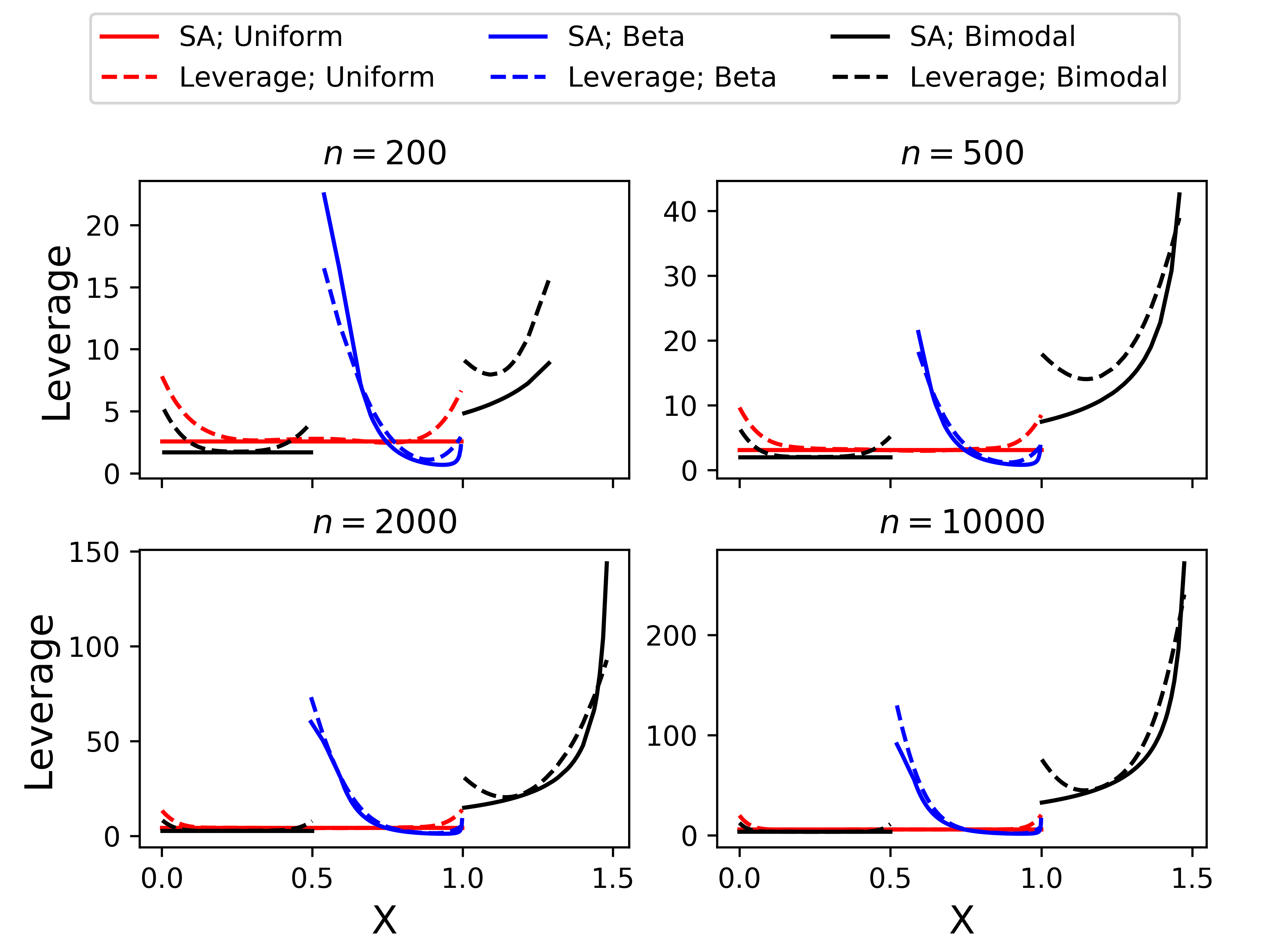}}
\vspace{.3in}
\caption{Statistical Leverage Score Approximation}
\label{lvg}
\end{figure}

In Figure~\ref{lvg}, dotted curves correspond to the rescaled leverage scores, while solid curves correspond to the equivalent kernel approximations.
We can observe our method provides good approximations to the rescaled leverage scores across all settings. 
In particular, Unif$[0,1]$ is the easiest case (red curves) due to its flat density, which meets Assumption~\ref{a.1} and~\ref{a.2} for almost all design points;
while for points with low density, such as those in the smaller cluster of the bimodal distribution and close to the boundary of Beta$(15, 2)$, the absolute error tends to be large, due to the leading constant $C_{x_i}$ in the error bound in Theorem~\ref{thm:main_result}.
Moreover, the relative approximation error has a clear tendency to decrease as the sample size increases, which is consistent with our theory.

We also quantitatively study the accuracy of the leverage scores obtained by different methods in the last section.
Each algorithm is tested on \texttt{RadiusQueriesCount} \citep{savva2018explaining, anagnostopoulos2018scalable}(denoted by RQC), 
\texttt{HTRU2} \citep{lyon2016fifty}, 
and \texttt{CCPP} \citep{tufekci2014prediction, tufekci2012local}, the datasets downloaded from the UCI ML Repository \citep{Dua:2019}. 
Those datasets contain 10000, 17898, and 9568 data points respectively,
which are at the limit of our computational feasibility 
(it requires $\m O(n^3)$ time and $\m O(n^2)$ space to exactly compute leverage scores).
We begin by normalizing the datasets before constructing the kernel matrix using \matern kernel ($\nu=0.5$).
Each method is then used to approximate the leverage scores $\{\tilde \ell_i\}_1^n$.
The sampling probability $\tilde q_i$ is obtained as $\tilde \ell_i / (\sum_1^n \tilde \ell_i)$ 
(also denoting $q_i = \ell_i / (\sum_1^n \ell_i)$).
The accuracy of each method is measured by the average of the ratios 
$\{r_i \defeq \tilde q_i / q_i \}_1^n$ (R-ACC).
The complete setting for this experiment can be found in Section~\ref{sec:exp_details} in the appendix.

In Table~\ref{tab:lvg_approx} we report the runtime, mean R-ACC $\bar r$, and the 5th / 95th quantile of R-ACC, averaged over 10 replicates.
We notice that, regarding the leverage approximation, our method provides the most accurate leverage approximation (in terms of mean R-ACC), and is more efficient than other methods on the benchmark datasets, 
which matches the complexity analysis.

\subsection{Additional empirical results}

In Section~\ref{Sec:hbe} in the appendix, we further provide some empirical results to compare different approximation methods for increasing input dimension $d$.
In short, under the certain setting the prediction accuracy of all the methods will greatly deteriorate due to the curse of dimensionality,
and the classical \nystrom method with uniform sampling will be preferred as leverage-based sampling cannot bring many benefits to the statistical performance.

\section{Conclusion and future work}

We propose a new method to estimate the leverage scores in kernel ridge regression for fast \nystrom approximation when a stationary kernel is used.
Theoretical results are also provided to guarantee the high accuracy of our estimation.
In particular, we show that under the mild conditions the leverage scores induced by a \matern empirical kernel matrix can be estimated in $\wt{\m O}(n)$ time, where $n$ is the size of input samples.

A direct further development of our current work is the extension of our theory to other stationary kernels, such as Gaussian kernels and exponential kernels.
Other related questions include the performance guarantees when the new leverage estimation method is applied to kernel methods for other machine learning problems, for example, kernel $k$-means and kernel PCA.
It will also be interesting to follow the heuristic procedure in our method to analyze other kernel models involving linear smoothers,
and seek the possibility to accelerate those models.
The results in this work also shed new light on the relevance of the ``equivalent kernel" \citep{yang2017frequentist} and the regularized Christoffel function \citep{pauwels2018relating} mentioned in Section~\ref{Sec:contribution}.

\acks{This work is supported by NSF grant DMS-1810831.}

\bibliography{fast_lvg_ref}

\appendix

The outline of the appendix is stated as follows.
First, some useful results along with some proofs of the preliminaries and the main results are provided in Section~\ref{sec:main_results}.
In Section~\ref{sec:exp_details} we report the detailed settings of the experiments in the main paper.
In Section~\ref{sec:ek_ppt}, we further analyze the properties of our leverage score approximation, which mirror the behavior of the true statistical leverage scores.
To numerically accelerate the computation in the multivariate case, we simplify the multiple integral for obtaining the leverage score approximation to a single integral in Section~\ref{sec:polar}.
Besides, we show the error caused by density estimation is negligible compared with the total error in approximating statistical leverage scores in Section~\ref{sec:density_estimation}.
Finally, we provide some technical facts regarding multivariate integration in Section~\ref{Sec:int_by_parts}.

\section{USEFUL FACTS}
\label{sec:main_results}

\subsection{Fourier Transform}
Use $L_p(\mb R^d) = \big\{f:\, \mb R^d \to \mb R,\, \int_{\mb R^d} \big|f(x)\big|^p\,\dd x<\infty \big\}$ to denote the space of all $L_p$ integrable functions over $\mb R^d$ for $p\geq 1$.
For any function $f \in L_1(\mb R^d)$, we use $\ms F[f]$ and $\ms F^{-1}[f]$ to denote its Fourier transform and its inverse Fourier transform, which is given by
\begin{align*}
\ms F[f](s) &= \int_{\mb R^d} f(x)\,e^{-2\pi \sqrt{-1} \langle x, s \rangle}\,\dd x, \quad \mbox{for all $s \in \mb R^d$}, \\
 \ms F^{-1}[f](x) &= \int_{\mb R^d} f(s)\, e^{2\pi \sqrt{-1} \langle x, s \rangle}\,\dd s,~~ \quad \mbox{for all $x \in \mb R^d$}.
\end{align*}

A useful property of the Fourier transform is the Parseval's identity.
\begin{theorem}[Parseval's identity]\label{thm:Parseval}
For any $f\in L_2(\mb R^d)$, the following identity holds
\begin{align*}
\int_{\mb R^d} \big|f(x)\big|^2\,\dd x = \int_{\mb R^d} \big|\ms F[f](s)\big|^2\,\dd s.
\end{align*}
\end{theorem}
Besides, Fourier transform is closely related to kernels.
For any PSD stationary kernel $K:\,\mb R^d\times\mb R^d\to \mb R$, 
by Brochner's Theorem, it would be a Fourier transform of a Borel measure.
For simplicity we abuse the notation by using $K(x)$ to mean $K(y,\, y+x), \, \forall y \in \mb R^d$, since $K(y,\, y+x)$ does not depend on the specific choice of $y$.
Specifically, the \matern kernel $K_{\alpha}$ with smoothness parameter $\nu = \alpha - d/2>0$ can be equivalently defined through its Fourier transform~\citep[p. 84]{rasmussen2003gaussian} as
\begin{align*}
m_\alpha(s):\,= \ms F[K_\alpha](s) = \int_{\mb R^d} K_\alpha(x)\, e^{-2\pi \sqrt{-1} \langle x, s \rangle}\,\dd x = C_\alpha(1+D_\alpha \|s\|^2)^{-\alpha}, \quad\forall s\in\mb R^d,
\end{align*}
where $C_\alpha$ and $D_\alpha$ are some constants only dependent on $\alpha$. 
(To simplify the statement of the theory, we would simply take $C_\alpha=D_\alpha=1$ when later discussing the asymptotic properties.)
Throughout this appendix, we focus on the case in which the \matern kernel is used, 
since its theoretical properties have been well studied, and the proof is easy to be extended to other stationary kernels.

\subsection{RKHS Associated with the \matern Kernel—Proof of Theorem~\ref{Eqn:RKHS_Fourier} in the Main Paper}
\label{sec:thm2}

The following theorem characterizes the RKHS $\mb H_\alpha$ associated with the \matern kernel through its Fourier transform. The main body of the proof comes from the slides of \citet{fukumizu2008elements}.
\begin{lemma}[Fourier representation of RKHS]\label{lem:RKHS_norm}
For any $f, g \in \mb H_\alpha$, we have
\begin{align*}
\|f\|_{\mb H_\alpha}^2 = \int_{\mb R^d} \frac{\big|\ms F[f](s)\big|^2}{m_\alpha(s)}\, \dd s, \quad\mx{and}\quad  \langle f,\, g\rangle_{\mb H_\alpha} = \int_{\mb R^d} \frac{\ms F[f](s)\cdot \overline{\ms F[g](s)}}{m_\alpha(s)}\, \dd s
\end{align*}
\end{lemma}


\begin{proof}

Consider a measure space $(\mathbb{R^d}, \mathcal{B}, \mu$), where $d \mu = m_\alpha(s) ds$, 
and $m_\alpha(s)$ is the spectral density of the invariant \matern kernel with smoothness parameter $\nu = \alpha - \frac{d}{2}$. 
By referring to Section~\ref{sec:heuristic} in our main paper, we can check the function $m_\alpha(s) \in L^\infty$ is differentiable and positive everywhere.
Based on the measure $\mu$, we define a function space $\mathbb{G} = L^2(\mathbb{R}^d, \mu) \equiv {\{F: \mb R^d \rightarrow \mathbb{C}; \int_{\mb R^d} |F|^2 d \mu < \infty \}}$, 
with the inner product $\langle F, G \rangle_\mathbb{G} \defeq \int F \overline{G} d \mu$. 
Here $\mathbb{C}$ is the set of all complex numbers.
We can observe the form is quite similar to the frequency domain in Fourier Transform. 
To construct the RKHS of interest, we further define a function $H(s; x) = \exp(-2 \pi \sqrt{-1} \dotp{x}{s})$ and a map $\ms M(\cdot): L^2(\mb R^d, \mu) \rightarrow \mb T$, similar to inverse Fourier Transform, given by
\begin{align*}
\ms M(F)(x) \defeq \int F(s) \overline{H(s; x)} d \mu
= \int F(s) \exp(2 \pi \sqrt{-1} \dotp{x}{s}) \dd \mu, \quad \forall x \in \mb R^d
\end{align*}
where $\mb T$ is the space of all functions over $\mb R^d$ with the pointwise-convergence topology, i.e. $f_n \rightarrow f \Leftrightarrow f_n(x) \rightarrow f(x), \forall x \in \mathbb{R}^d$.

Now we are able to define a new function space $\mathbb{H} \defeq \{f \in \mb T; \,  \exists F \in L^2(\mathbb{R}^d, \mu), f = \ms M(F)\}$,
and equip it with the inner product $\langle f, g\rangle_\mathbb{H} \defeq \langle F, G \rangle_\mathbb{G}$, where $F, G$ satisfy $f = \ms M(F), g = \ms M(G)$. 
We first need to show $\mathbb{H}$ is an RKHS. 
We can check for all $f \in \mathbb{H}$,
\begin{align*}
f(x) = \langle F, H(\cdot; x) \rangle_\mathbb{G} = \langle f, \ms M(H(\cdot; x)) \rangle_\mathbb{H}
\end{align*}
Also, the reproducing kernel of $\mathbb{H}$ would be:
\begin{align*}
K(x, y) &= \langle \ms M(H(\cdot; x)), \ms M(H(\cdot; y)) \rangle_\mathbb{H} = \langle H(\cdot; x), H(\cdot; y)\rangle_\mathbb{G} \\
&= \int \exp(-2 \pi \sqrt{-1} \dotp{x}{s}) \exp(2 \pi \sqrt{-1} \dotp{y}{s}) d \mu \\
&= \int \exp(2 \pi \sqrt{-1} \dotp{y - x}{s}) m_\alpha(s) ds \\
&= K_\alpha(y - x)
\end{align*}
where $K_\alpha$ is the invariant \matern kernel with smoothness parameter $\alpha$.
We can confirm $\mathbb{H}$ is exactly the RKHS induced by a \matern kernel $K_\alpha$.

To complete the proof, we still need to find the form of $\langle f, g \rangle_\mathbb{H}$. 
Note by the facts $m_\alpha(s) \in L^\infty$ and $F \in L^2$, we can infer $F(s) m_\alpha(s) \in L^2$ as $\int (F(s) m_\alpha(s))^2 \dd s \leq \|m_\alpha\|_\infty^2 \int F^2(s) \dd s$.
Using the Fourier isometry of $L^2$~\citep[Theorem~7.61]{adams2003sobolev}, we obtain ${F(s) m_\alpha(s) = \ms F[f](s)}$; i.e., $F(s) = \ms F[f](s) / m_\alpha(s)$. 
Finally, by the definition of the inner product, for $f = \ms M(F)$ and $g = \ms M(G)$,
\begin{align*}
\langle f, g\rangle_\mathbb{H} &= \langle F, G\rangle_\mathbb{G} = \int \frac{\ms F[f](s)}{m_\alpha(s)} \overline{\frac{\ms F[g](s)}{m_\alpha(s)}} m_\alpha(s) ds \\
&= \int \frac{\ms F[f](s) \overline{\ms F[g](s)}}{m_\alpha(s)} ds
\end{align*}
in which the third equality relies on the fact the $m_\alpha(s)$ is real, and its conjugate is the same as itself.
\end{proof}

\subsection{Embedding Inequalities}
\label{Sec:embedding}

Let $\alpha$ be an integer. Following the notation of the book~\citep{adams2003sobolev}, for any $p\geq 1$ and subset $\Omega \subset \mb R^d$, we use the notation $W^{\alpha, p}(\Omega)$ to denote the Sobolev space as a set of functions $u$ in $L_p(\Omega)$ such that $u$ and its weak derivatives up to total order $\alpha$ have a finite $L_p$ norm. 
With this definition, the Sobolev space admits a norm and a seminorm
\begin{align}\label{eqn:S_norm}
\|u\|_{\alpha,p,\Omega} &= \bigg(\sum_{|\mathbf k| \leq \alpha} \|D^{\mathbf k} u\|_{p,\Omega}^p\bigg)^{\frac{1}{p}}
= \bigg(\sum_{|\mathbf k| \leq \alpha} \int_{\Omega} \big|D^{\mathbf k} u(t)\big|^p\,\dd t\bigg)^{\frac{1}{p}}, \\
|u|_{\alpha, p, \Omega} &= \bigg(\sum_{|{\mathbf k}| = \alpha} \int_{\Omega} \big|D^{\mathbf k} u(t)\big|^p\,\dd t\bigg)^{\frac{1}{p}}.
\end{align}
where ${\mathbf k} = (k_1, \cdots, k_d)$ is a multi-index, and $D^{\mathbf k} u = \frac{\partial^{|{\mathbf k}|} u}{\partial^{k_1} x_1 \cdots \partial^{k_d} x_d}$.
To more precisely describe the mixed derivative, we additionally define some notations here for future use:
\begin{itemize}
\item $\mathbf A$ is a subset of $[d]$, $- \mathbf A \defeq ([d] - \mathbf A)$,
\item $u = \mathbf I_{\mathbf A} \in \{0, 1\}^d$ is a vector s.t. $u_i = 1, \forall i \in \mathbf A; u_i = 0, \forall i \in - \mathbf A$,
\item a vector $u = \mu(\mathbf A_{-1}, \mathbf A_{+1}) \in \{-1, 0, 1\}^d$ satisfies $u_i = -1, \forall i \in \mathbf A_{-1}; u_i = 1, \forall i \in \mathbf A_{+1}$,
\item $x_{\mathbf A} \defeq (x_i)_{i \in \mathbf A}$, 
and $g(x_{\mathbf A}; x_{- \mathbf A})$ represents $g(u)$ where $u_{\mathbf A} = x_{ \mathbf A}$ are the variables and $u_{-\mathbf A} = x_{- \mathbf A}$ are taken as the parameters fixed in the integration,
\item $C(y, \delta) \defeq \{(x_1, x_2, \dots, x_d); x_i \in [y_i-\delta, y_i+\delta], \forall i \in [d]\}$ is a cube centered at $y$,
\item $C^{\mathbf A}(y, \delta) = C(y_{\mathbf A}, \delta)$ is the marginal cube of $C(y, \delta)$ defined as $\{x_{\mathbf A}; x_i \in (y_i-\delta, y_i+\delta), \forall i \in \mathbf A\}$.
\end{itemize}

We will primarily work with the case $p=2$ and $\Omega$ being a connected domain. 
For any $u \in \mb H$, by using the fact that $\ms F[D^{\mathbf k} u](s) = (\prod_{i=1}^d (2\pi \sqrt{-1} s_i)^{k_i}) \cdot \ms F[u](s)$ and the Parseval's identity in Theorem~\ref{thm:Parseval}, we obtain
\begin{align*}
\|u\|_{\alpha,2,\mb R^d}^2 = \int_{\mb R^d} \sum_{|{\mathbf k}| \leq \alpha} \big| \prod_{i=1}^d (2\pi \sqrt{-1} s_i)^{k_i}\,  \ms F[u](s)\big|^2\,\dd s.
\end{align*}
Since there exist constants $(C_1,C_2)$ such that $C_1 (1 + \|s\|^2)^\alpha \leq \big (\sum_{|{\mathbf k}| \leq \alpha} \prod_{i=1}^d s_i^{k_i} \big )^2 \leq C_2 (1 + \|s\|^2)^\alpha$ holds for all $s \in \mb R^d$, by Lemma~\ref{lem:RKHS_norm} we can further deduce that
\begin{align}\label{Eqn:norm_equv}
C_1 \|u\|_{\mb H_\alpha}^2 \leq \|u\|_{\alpha,2,\mb R^d}^2 \leq C_2 \|u\|_{\mb H_\alpha}^2
\end{align}
holds for any $u\in\mb H_\alpha$, the RKHS associated with the \matern kernel with smoothness index $\nu (= \alpha - d/2)$.

We first invoke the following special case of interpolation theorem of Sobolev space $W^{\alpha, p}(\Omega)$ \citep[Theorem~5.12]{adams2003sobolev}.

\begin{theorem}[Interpolation inequality]\label{thm:Sobolev_int}
For any integer $0 \leq k \leq \alpha$, there exist two constants $(c_0, K)$ only depending on $\alpha$, such that for any $u \in W^{\alpha,2}(\Omega)$ and any $\varepsilon \in (0,c_0)$,
\begin{align*}
|u|_{k, 2, \Omega} \leq K \big(\varepsilon^{\alpha-k}\, |u|_{\alpha, 2, \Omega} + \varepsilon^{-k}\|u\|_{2, \Omega}\big),
\end{align*}
where for any function $g$, $\|g\|^2_{2,\Omega} = \int_\Omega g^2(t)\,\dd t$.
\end{theorem}

We will also use the following generalization of Gagliardo–Nirenberg interpolation inequalities to bound the sup-norm, which can be viewed as an extension of the above interpolation inequality to the sup-norm. 
\begin{theorem}[Sup-norm interpolation inequality]\label{thm:sup_int}
There exist two universal constants $(c_1, K_1)$, such that for any $u \in W^{\alpha, 2}(\Omega)$ and any $\varepsilon \in (0, c_1)$, 
\begin{align*}
\|u\|_{\infty, (1-\varepsilon^2) \Omega}:\,=\sup_{t \in (1-\varepsilon^2) \Omega}|u(t)| \leq K_1\big(\varepsilon^{-d}\, \|u\|_{2, \Omega} + \varepsilon^{2\alpha - d}\, |u|_{\alpha, 2, \Omega}\big).
\end{align*}
\end{theorem}
\begin{proof}
From the Gagliardo–Nirenberg interpolation inequalities~\citep{brezis2018gagliardo}, we have 
\begin{align*}
\|g\|_{\infty, \Omega} \leq C (\|g\|_{2, \Omega} + |g|_{\alpha, \frac{d}{\alpha}, \Omega}), \quad\forall \ g \in W^{\alpha,2}(\Omega),
\end{align*}
where $C$ is some universal constant. 
In fact, the last term could be further bounded by $C_0 |g|_{\alpha, 2, \Omega}$
since $W^{\alpha,2} = H_\alpha, \alpha > \frac{d}{2}, \frac{d}{\alpha} < 2$, and the domain $\Omega$ is bounded. 

Now for each fixed point $y_0 \in (1-\varepsilon^2) \Omega$, we can obtain by the preceding display with $g(t) = u(y_0 + \varepsilon^2 t)$ in the above that
\begin{align*}
\|u\|_{\infty, B(y_0, \varepsilon^2)} \leq \varepsilon^{-d}\,\|u\|_{2, B(y_0, \varepsilon^2)} + C_0 \, \varepsilon^{2\alpha - d} \, |u|_{\alpha, 2, B(y_0, \varepsilon^2)}
\leq C \varepsilon^{-d}\,\|u\|_{2,\Omega} + C_0 \, \varepsilon^{2\alpha - d}\, |u|_{\alpha, 2, \Omega},
\end{align*}
where we have used the fact that $B(y_0, \varepsilon^2) \subset \Omega$ for any $y_0 \in (1-\varepsilon^2) \Omega$. 
Finally, the claimed inequality follows by the above derivation and the fact that
\begin{align*}
\|u\|_{\infty, (1-\varepsilon^2) \Omega} = \sup_{y_0 \in (1-\varepsilon^2) \Omega} \|u\|_{\infty, B(y_0, \varepsilon^2)}
\end{align*}
\end{proof}

Theorem~\ref{thm:Sobolev_int} and~\ref{thm:sup_int} leads to the following lemma that we will repeatedly use in our proof.
Here we specifically consider a fixed point $x_0$ such that the density function $p(x)$ is uniformly bounded from below by $\frac12 \,p(x_0)>0$ for any $x$ satisfying $\|x - x_0\| < \delta(x_0)$ and some constant $\delta(x_0) > 0$ that may depend on $x_0$. 
Before entering the theorem, we define an RKHS norm $\|\cdot\|_{\lambda}$ as
\begin{align*}
\|f\|_\lambda^2 = \int_{\mb R^d} f^2(x) \, p(x)\,\dd x + h^{2\alpha} \|f\|_{\mb H_\alpha}^2,
\end{align*}
and its localized truncation $\|f\|_{x_0, \lambda}$ as 
\begin{align*}
\|f\|_{x_0,\lambda}^2 \defeq \int_{C(x_0,\delta(x_0))} f^2(x) \, p(x)\,\dd x + h^{2\alpha} \|f\|_{\mb H_\alpha}^2.
\end{align*}

\begin{theorem}[Local interpolation for RKHS]\label{thm:Local_embd}
Suppose {$\delta(x_0) \geq Ch \log (1/h)$} for some constant $C>0$.
Let $C(x_0,\delta(x_0))$ denote a cube centered at $x_0$ with edge length $2 \delta(x_0)$. 
If $C \log(1/h) > c_0^{-1}$ and $\sqrt{C \log(1/h)} > c_1^{-1}$, where $(c_0,c_1)$ are the constants in Theorems~\ref{thm:Sobolev_int} and~\ref{thm:sup_int}, then there exists a constant $K'$ such that
\begin{align*}
|f|_{k, 2, C(x_0,\delta(x_0))} \leq  K' h^{-k} \max\big\{1,\big(p(x_0)\big)^{-1/2}\big\}\,\|f\|_{x_0, \lambda},
\end{align*}
for any $f \in \mb H_\alpha$ and $k = 0, 1,\ldots,\alpha$. 
In addition, for each $k = 0, 1, \ldots \alpha$, there exists some constant $K''>0$ and $\varepsilon<c_1$ such that 
\begin{align*}
\|f\|_{k, \infty, (1-\varepsilon^2)C(x_0, \delta(x_0))} \leq K'' h^{-k-d/2} \max\big\{1,\big(p(x_0)\big)^{-1/2}\big\} \, \|f\|_{x_0,\lambda},
\end{align*}
Finally, for the case $|\mathbf A|=k, |\mathbf A'|=k'$, and $\mathbf A' \subseteq \mathbf A$, 
there is also a constant $K'''$ satisfying:
\begin{align*}
\Big( \int_{C^{\mathbf A}(x_0,\delta(x_0))}
\big| D^{\mathbf 1_{\mathbf A'}} f(x_{\mathbf A}; x_{- \mathbf A}) \big|^2 \dd x_{\mathbf A} \Big)^{\frac{1}{2}}
\leq K''' h^{-k'- (d-k)/2} \max\big\{1,\big(p(x_0)\big)^{-1/2}\big\} \, \|f\|_{x_0,\lambda},
\end{align*}
which generalizes the first inequality and provides a finer $L^2$ norm control.
\end{theorem}

\begin{proof}
When $k=\alpha$, the first inequality is obvious due to the the equivalence~\eqref{Eqn:norm_equv} between $\|\cdot\|_{\alpha,2,\mb R}$ and $\|\cdot\|_{\mb H_\alpha}$, and the fact that $\|f\|_{x_0,\lambda} \geq h^{\alpha} \|f\|_{\mb H_\alpha}$. 
Now let us consider $k \leq \alpha-1$.
Let $a=\delta(x_0)$ and $u(t) = f(x_0+at)$ for $t \in \Omega$. 
By applying the change of variable formula for integral and the chain rule for derivatives, we obtain (with $x = at$)
\begin{align*}
|u|^2_{k, 2, \Omega} 
&= \sum_{|{\mathbf j}| = k} \int_{\Omega} \big|D^{\mathbf j} u(t)\big|^2 \,\dd t 
= \sum_{|{\mathbf j}| = k} a^{2k-d} \int_{C(x_0, a)} \big|D^{\mathbf j} f(x)\big|^2\,\dd x \\
&= a^{2k-d} |f|^2_{k, 2, C(x_0, a)}, \quad\forall \ k = 0, 1, \ldots, \alpha.
\end{align*}
Combining this with Theorem~\ref{thm:Sobolev_int} and the definition~\eqref{eqn:S_norm} of Sobolev norm yields 
\begin{align*}
|f|_{k, 2, C(x_0, \delta(x_0))}^2
& \leq a^{-(2k - d)} \bigg( K \varepsilon^{-k} \big( \varepsilon^{\alpha} |u|_{\alpha, 2, \Omega} + \|u\|_{2, \Omega} \big) \bigg)^2 \\
&\leq 2 K^2\,(a\varepsilon)^{-2k} \bigg((a\varepsilon)^{2\alpha} |f|_{\alpha, 2, C(x_0, \delta(x_0))}^2
+ \int_{C(x_0, \delta(x_0))} \big|f(x)\big|^2\,\dd x\bigg) \\
&\leq 2 K^2\,(a\varepsilon)^{-2k} \bigg((a\varepsilon)^{2\alpha} |f|_{\alpha, 2, \mb R^d}^2
+ \int_{C(x_0, \delta(x_0))} \big|f(x)\big|^2\,\dd x\bigg).
\end{align*}
Using the condition that $p(x) \geq p(x_0)/2$ for each $x \in C(x_0, \delta(x_0))$, 
and the equivalence~\eqref{Eqn:norm_equv} between $\|\cdot\|_{\alpha, 2, \mb R^d}$ and $\|\cdot\|_{\mb H_\alpha}$, we further obtain by choosing $\varepsilon = h / \delta(x_0) \leq (C\log(1/h))^{-1} < c_0$ in the above that
\begin{align*}
|f|_{k, 2, C(x_0, \delta(x_0))}^2 
&\leq 2 K^2 h^{-2k} \Big(h^{2\alpha} \,\|f\|_{\mb H_\alpha}^2 
    + 2 \big(p(x_0)\big)^{-1} \, \int_{C(x_0, \delta(x_0))} \big|f(x)\big|^2\,p(x)\,\dd x\Big) \\
&\leq K'^2 h^{-2k} \max\big\{1,\big(p(x_0)\big)^{-1}\big\} \|f\|_{x_0,\lambda}^2,
\end{align*}
which yields the first claimed inequality.

To prove the second inequality, we will apply Theorem~\ref{thm:sup_int}. 
More specifically, we apply Theorem~\ref{thm:sup_int} with $u(t) = D^{\mathbf j} f(x_0 + at)$, $|{\mathbf j}|=k$, $a=\delta(x_0)$, $\Omega=C(x_0, \delta(x_0))$ and set $\varepsilon = \sqrt{h/a} \leq \sqrt{\frac{1}{C \log(1/h)}} < c_1$ to obtain
(with a change of variable formula for integration)
\begin{align*}
\|D^{\mathbf j} f\|_{\infty, (1-\varepsilon^2)\Omega} &= \|u\|_{\infty, (1-\varepsilon^2) C(0, 1)}
\leq K_1 \big(\varepsilon^{-d}\, \|u\|_{2, C(0, 1)} + \varepsilon^{2\alpha-2k-d} |u|_{\alpha-k, 2, C(0, 1)}\big) \\
&\leq\, K_1 \varepsilon^{-d} a^{-d/2}\, |f|_{k, 2, \Omega} 
    + K_1 \varepsilon^{2\alpha-2k-d} a^{\alpha-k-d/2}\, |f|_{\alpha, 2, \Omega} \\
&= K_1 h^{-d/2}\, |f|_{k, 2, \Omega}
    + K_1 h^{\alpha-k-d/2}\, |f|_{\alpha, 2, \Omega}.
\end{align*}
Now, we can obtain by combining the above with the first inequality of this theorem,
\begin{align*}
|f|_{k, \infty, (1-\varepsilon^2) \Omega} 
\lesssim h^{-k-d/2} \max \big\{1,\big(p(x_0)\big)^{-1/2}\big\} \, \|f\|_{x_0,\lambda}.
\end{align*}

To prove the final inequality, we will take $f(x_{\mathbf A}; x_{-\mathbf A})$ as a $k$-d function.
Analogously using the previous result, we have:
\begin{align*}
\int_{C^{\mathbf A}} \big| D^{\mathbf 1_{\mathbf A'}} f(x_{\mathbf A}; x_{- \mathbf A}) \big|^2 \dd x_{\mathbf A}
&\leq |f|_{k', 2, C^{\mathbf A}}^2 \\
&\leq K'^2 h^{-2k'} \big(\int_{C^\mathbf A} f^2(x_{\mathbf A}; x_{- \mathbf A}) \dd x_{\mathbf A} 
    + h^{2\alpha} |f(x_{\mathbf A}; x_{- \mathbf A})|_{\alpha, 2, C^{\mathbf A}}^2 \big)
\end{align*}
Then we take $D^{\mathbf g} f(x_{\mathbf A}; x_{- \mathbf A})$ as a $(d-k)$-d function over $C^{-\mathbf A}(x_0, \delta(x_0))$ ($g$ is any multi-index whose nonzero elements are in $A'$),
and utilize the intermediate result of the second inequality:
\begin{align*}
\|D^{\mathbf g} f(x_{\mathbf A}; x_{- \mathbf A})\|_{\infty, (1-\varepsilon^2) C^{-\mathbf A}} 
\lesssim h^{-(d-k)/2} |f|_{|\mathbf g|, 2, C^{-\mathbf A}}
+ h^{-(d-k)/2 + (\alpha - |\mathbf g|)} |f|_{\alpha, 2, C^{-\mathbf A}}.
\end{align*}
In that case,
\begin{align*}
\int_{C^\mathbf A} f^2(x_{\mathbf A}; x_{- \mathbf A}) \dd x_{\mathbf A} 
&\lesssim \int_{C^\mathbf A} h^{-(d-k)} \|f(x_{\mathbf A}; x_{- \mathbf A})\|_{2, C^{-\mathbf A}}^2 
+ h^{-(d-k) + 2\alpha} |f|_{\alpha, 2, C^{-\mathbf A}}^2 \dd x_{\mathbf A} \\
&\lesssim h^{-(d-k)} \|f\|_{2, \Omega}^2 + h^{-(d-k) + 2\alpha} |f|_{\alpha, 2, \Omega}^2 \\
&\lesssim h^{-(d-k)} \max \big\{1,\big(p(x_0)\big)^{-1}\big\} \, \|f\|_{x_0,\lambda}^2,
\end{align*}
The result for $h^{2\alpha} |f(x_{\mathbf A}; x_{- \mathbf A})|_{\alpha, 2, C^{\mathbf A}}^2$ could be analogously obtained.
Combining the pieces together, we have
\begin{align*}
\int_{C^{\mathbf A}} \big| D^{\mathbf 1_{\mathbf A'}} f(x_{\mathbf A}; x_{- \mathbf A}) \big|^2 \dd x_{\mathbf A} 
\lesssim h^{-2k'-(d-k)} \max \big\{1,\big(p(x_0)\big)^{-1}\big\} \, \|f\|_{x_0,\lambda}^2
\end{align*}
\end{proof}

\subsection{Leverage Score Approximation—Proof of Theorem~\ref{thm:main_result} in the Main Paper}

\begin{proof}
Let $F$ denote the limiting cumulative distribution function of $F_n$, and $p$ the density function associated with $F$. 
Recall that the rescaled leverage approximation $\wt{K}_\lambda(x,x_0)$ is the minimizer of the following local population level functional
\begin{align*}
A_{x_0}(f) = \frac{p(x_0)}{2} \int_{\mb R^d} f^2(x) \,\dd x + \frac{\lambda}{2} \|f\|_{\mb H_\alpha}^2 -f(x_0),
\end{align*}
such that the following identity holds for each function $u\in\mb H_\alpha$, which corresponds to setting the Gateaux derivative $DA_{x_0}$ of $A_{x_0}$ at $\wt K_{x_0}$ to be the zero operator,
\begin{align*}
DA_{x_0}(\wt K_{x_0})(u) &= p(x_0) \int_{\mb R^d} \wt K_{x_0}(x)\, u(x)\,\dd x + \lambda\,\langle \wt K_{x_0},\,u\rangle_{\mb H_\alpha} - u(x_0) = 0, \\
\mx{or} \quad \wt K_{x_0}(x):\,=\wt K_\lambda(x,\, x_0)& = \ms F^{-1}\bigg[ \frac{1}{p(x_0) + h^{2\alpha}\, (1+\|s\|^2)^\alpha}\bigg](x-x_0),\quad\forall x \in \mb R^d.
\end{align*}

The rescaled leverage function $G_{x_0}$ is instead the minimizer of the empirical functional $A_{n,x_0}$,
and thus the Gateaux derivative $DA_{n, x_0}$ at point $G_{x_0}$ should be 0 since $G_{x_0}$ is the optimal function for the functional. 
Using that fact,
\begin{align*}
DA_{n,x_0}(\wt K_{x_0})(\wt u) &= \{DA_{n,x_0}(\wt K_{x_0}) - DA_{n,x_0}(G(\cdot,\, x_0))\}(\wt u) \\
&= D^2 A_{n,x_0}(G(\cdot,\, x_0))(\wt K_{x_0} - G(\cdot,\, x_0), \wt u)
\end{align*}
The last equality holds due to the definition of second order functional derivative.
Note the key identity that $D^2 A_{n,x_0}(G(\cdot,\, x_0))(\wt u, \wt u) = \|\wt u\|_{n,\lambda}^2$.
By choosing $u = \wt u \defeq \wt K_{x_0} - G(\cdot,\, x_0)$, we would further have ($u, \wt u$ would be used interchangeably from now on)
\begin{align*}
D A_{n,x_0}(\wt K_{x_0})(u) &= \|\wt u\|_{n,\lambda}^2 = \int_{\mb R^d} \wt u^2(x) \,\dd F_n(x) + \lambda\,\|\wt u\|_{\mb H_\alpha}^2,
\end{align*}
and our task somewhat reduces to bounding the term above $D A_{n,x_0}(u) = \|\wt K_{x_0} - G(\cdot,\, x_0)\|_{n,\lambda}^2$.
To do that, we can expand the expression $DA_{n,x_0}(\wt K_{x_0})(u)$:
\begin{align*}
&\,DA_{n,x_0}(\wt K_{x_0})(u) = \int_{\mb R^d} \wt K_{x_0}(x)\,\dd F_n(x) +  \lambda\,\langle \wt K_{x_0},\,u\rangle_{\mb H_\alpha} - u(x_0) \\
&= \underbrace{DA_{x_0}(\wt K_{x_0})(u)}_{=0} + \underbrace{\int_{\mb R^d} \wt K_{x_0}(x)\,u(x)\,\dd \big(F_n(x)-F(x)\big)}_{=\,: I_1} + \underbrace{\int_{\mb R^d} \wt K_{x_0}(x)\, u(x)\,\big(p(x)-p(x_0)\big)\,\dd x}_{=\,:I_2},
\end{align*}
and bound the last two terms separately.

Using Lemma~\ref{thm:int_by_parts}, as $u = \wt u$ vanishes at infinity, we have
\begin{align*}
I_1 = (-1)^d \int_{\mb R^d} (F_n(x) - F(x)) \frac{\partial^d}{\partial x_1 \partial x_2 \cdots \partial x_d}
\big( \wt K_{x_0}(x)\,u(x) \big) \dd x
\end{align*}
The term $|I_1|$ can be correspondingly bounded as
\begin{align*}
|I_1| &\leq \tau(n) \int_{\mb R^d} \big| \frac{\partial^d}{\partial x_1 \partial x_2 \cdots \partial x_d} \big( \wt K_{x_0}(x)\,u(x) \big) \big| \dd x \\
&\leq \tau(n) \sum_{\mathbf k_1 \sqcup \mathbf k_2 = [d]} \int_{\mb R^d} \big| D^{\mathbf k_1} \wt K_{x_0}(x) D^{\mathbf k_2} u(x) \big| \dd x.
\end{align*}

Using Lemma~\ref{Lem:EqKernel_property_matern}(\ref{kernelScale}) about the exponential decay on $\wt K_{x_0}$ and its derivatives and the local embedding inequalities in Theorem~\ref{thm:Local_embd}, we obtain
\begin{align*}
&\int_{\mb R^d} \big| D^{\mathbf k_1} \wt K_{x_0}(x) D^{\mathbf k_2} u(x) \big| \dd x 
\leq \int_{C_{x_0,\,\delta(x_0)}} \big| D^{\mathbf k_1} \wt K_{x_0}(x) \big| \,\big| D^{\mathbf k_2} u(x) \big|\, \dd x \\
&\qquad \qquad\qquad \qquad
+ |u|_{|\mathbf k_2|, 2, \mb R^d} \Big( \int_{C_{x_0,\,\delta(x_0)}^c} \big| D^{\mathbf k_1} \wt K_{x_0}(x) \big|^2 \dd x \Big)^{\frac{1}{2}} \\
\overset{(i)}{\leq}& \Big(\int_{\mb R^d} \big|h^{-|\mathbf k_1|} (h^d + h^{-d}) \,e^{-C_2\,\|x-x_0\|/h}\big|^2\,\dd x\Big)^{1/2} \cdot |u|_{|\mathbf k_2|, 2, C_{x_0, \delta(x_0)}} \\
&\qquad \qquad\qquad \qquad
 + |u|_{|\mathbf k_2|, 2, \mb R^d} \Big( \int_{\|x-x_0\|\geq Ch\log(1/h)} \big|h^{-|\mathbf k_1|} (h^d + h^{-d}) \,e^{-C_2\,\|x-x_0\|/h}\big|^2 \dd x \Big)^{\frac{1}{2}}
\end{align*}
where step (i) follows by the Cauchy-Schwarz inequality and the assumption that $\delta(x_0) \geq Ch\log(1/h)$. Further bound is given as
\begin{align*}
&\int_{\mb R^d} \big| D^{\mathbf k_1} \wt K_{x_0}(x) D^{\mathbf k_2} u(x) \big| \dd x \\
\lesssim& h^{-d/2-|\mathbf k_1|}\, |u|_{|\mathbf k_2|, 2, C_{x_0, \delta(x_0)}} 
    + \log^{\frac{d-1}{2}}(1/h) h^{C_2C-d/2-|\mathbf k_1|}\, |u|_{|\mathbf k_2|, 2, \mb R^d}  \\
\overset{(ii)}{\lesssim}& h^{-d/2-|\mathbf k_1|-|\mathbf k_2|}\,\max\big\{1,\big(p(x_0)\big)^{-1/2}\big\}\, \|u\|_{x_0,\lambda} 
    + \log^{\frac{d-1}{2}}(1/h) h^{C_2C-d/2-|\mathbf k_1|}\, |u|_{|\mathbf k_2|, 2, \mb R^d},
\end{align*}
where step (ii) uses the first inequality in Theorem~\ref{thm:Local_embd} with $k=|\mathbf k_2|$.
The next bound is derived as,
\begin{align*}
&\int_{\mb R^d} \big| D^{\mathbf k_1} \wt K_{x_0}(x) D^{\mathbf k_2} u(x) \big| \dd x \\
\lesssim& h^{-3d/2}\,\max\big\{1,\big(p(x_0)\big)^{-1/2}\big\}\, \|u\|_{x_0,\lambda} + \log^{\frac{d-1}{2}}(1/h) h^{C_2C-d/2-|\mathbf k_1|}\, |u|_{|\mathbf k_2|, 2, \mb R^d} \\
\lesssim& h^{-3d/2}\,\max\big\{1,\big(p(x_0)\big)^{-1/2}\big\}\, \|u\|_{x_0,\lambda} + \log^{\frac{d-1}{2}}(1/h) h^{C_2C-d/2-|\mathbf k_1|-\alpha} \|u\|_{x_0,\lambda},
\end{align*}
in which the last step utilizes the fact that $\|u\|_{1, 2,\mb R} \leq \|u\|_{\mb H_\alpha} \leq h^{-\alpha}\,\|u\|_{x_0,\lambda}$. 
For $C > \alpha / C_2$, we can finally obtain

\begin{align*}
|I_1| &\lesssim \tau(n) h^{-3d/2} \,\max\big\{1,\big(p(x_0)\big)^{-1/2}\big\}\,\|u\|_{x_0,\lambda}.
\end{align*}

Similarly, by using the Lipschitz property of the density function $p$ as $|p(x) - p(x_0)| \leq \min \big\{2C_p, {L_{x_0}\|x-x_0\|} \big\}$ (where $C_p =\sup_{x}|p(x)|$ and $L_{x_0}$ is the local Lipschitz constant of $p$ around $x_0$), the exponential decay on $\wt K_{x_0}$ and the local embedding inequalities in Theorem~\ref{thm:Local_embd} with $k=0$, we obtain 
\begin{align*}
|I_2| & \lesssim \int_{C_{x_0,\,\delta(x_0)}} \big|\wt K_{x_0}(x)\big| \cdot \|x-x_0\| \cdot |u(x)|\, \dd x 
+ \|u\|_{\infty, \mb R^d} \int_{C_{x_0,\,\delta(x_0)}^c} \big|\wt K_{x_0}(x)\big|\, \dd x \\
&\lesssim \Big(\int_{\mb R^d} \big|(h^{d}+h^{-d})\,e^{-C_2\,\|x-x_0\|/h}\,\|x-x_0\| \big|^2\,\dd x\Big)^{1/2}\cdot \|u\|_{2,C_{x_0,\delta(x_0)}} \\
&\qquad \qquad\qquad \qquad\qquad \qquad
 + \|u\|_{\infty, \mb R^d} \int_{\|x-x_0\|\geq Ch\log(1/h)} \big|(h^{d}+h^{-d})\,e^{-C_2\,\|x-x_0\|/h}\big|\,\dd x\\
&\lesssim h^{-d/2+1}\, \|u\|_{2, C_{x_0,\delta(x_0)}} + h^{C_2C}\, \|u\|_{\infty, \mb R^d} \\
&\lesssim h^{-d/2+1}\,\max\big\{1,\big(p(x_0)\big)^{-1/2}\big\}\, \|u\|_{x_0, \lambda} + h^{C_2 C}\, \|u\|_{\infty, \mb R^d}.
\end{align*}

Putting pieces together, we obtain
\begin{align*}
\big|DA_{n,x_0}(\wt K_{x_0})(u)\big| \lesssim h^{C_2C}\, \|u\|_\infty + \max\big\{1,\big(p(x_0)\big)^{-1/2}\big\}\,\big(\tau(n) \, h^{-3d/2} + h^{-d/2+1}\big)\,\|u\|_{x_0,\lambda}.
\end{align*}


Now we return back to the right hand side of the identity $DA_{n,x_0}(\wt K_{x_0})(u) = \|\wt u\|_{n,\lambda}^2$.
Since $F_n$ is nondecreasing, we have the following bound,
\begin{align*}
\int_{\mb R^d} \wt u^2(x) \,\dd F_n(x) &\geq \int_{C(x_0,\delta(x_0))} \wt u^2(x) \,\dd F_n(x)\\
& =\int_{C(x_0,\delta(x_0))} \wt u^2(x) \,\dd F(x) +\int_{C(x_0,\delta(x_0))} \wt u^2(x) \,\dd \big( F_n(x) - F(x)\big).
\end{align*}
Therefore, by the definition of the localized norm $\|\cdot\|_{x_0,\lambda}$, we have
\begin{align*}
\|\wt u\|_{n,\lambda}^2 \geq \|\wt u\|_{x_0,\lambda}^2 + \underbrace{\int_{C(x_0,\delta(x_0))} \wt u^2(x) \,\dd \big( F_n(x) - F(x)\big)}_{=\,:I_3}.
\end{align*}
By applying the Lemma~\ref{thm:int_by_parts} again (note $\wt u$ and $\wt u^2$ are infinitely differentiable), the second term $I_3$ can be bounded as (some terms are hidden)
\begin{align*}
|I_3| & \lesssim \big\|\wt u^2(x)\,\big(F_n(x) - F(x)\big)\big\|_{\infty, C_{x_0, \delta(x_0)}} + \dots \\
&\quad\quad\quad\quad\quad + \big\|F_n(x) - F(x)\big\|_{\infty, C_{x_0, \delta(x_0)}} \int_{C_{x_0,\delta(x_0)}} |\frac{\partial^d}{\partial x_1 \partial x_2 \cdots \partial x_d} (\wt u^2(x))|\,\dd x.
\end{align*}
Now by applying the first and the second inequality in Theorem~\ref{thm:Local_embd}, and the Cauchy-Schwarz inequality, the sum of the two terms above can be bounded up to a constant by 
\begin{align*}
\max\big\{1,\big(p(x_0)\big)^{-1}\big\}\,\tau(n)\, h^{-d}\,\|\wt u\|^2_{x_0,\lambda}.
\end{align*}
Putting all the pieces together, we can reach
\begin{align*}
&\big(1- c \tau(n) \,\max\big\{1,\big(p(x_0)\big)^{-1}\big\}\,h^{-d}\big)\, \|\wt u\|_{x_0,\lambda}^2 \\
& \qquad\qquad \leq c'\, h^{C_2C}\, \|\wt u\|_\infty + c'\,\max\big\{1,\big(p(x_0)\big)^{-1/2}\big\}\,\big(\tau(n) h^{-3d/2} + h^{-d/2+1}\big)\,\|\wt u\|_{x_0,\lambda}.
\end{align*}
It is easy to verify directly that we always have the crude bound $\|\wt u\|_\infty \lesssim n$, so by choosing constant $C$ sufficiently large $h^{C_2C} n$ is decreasing, we can obtain from the above that
\begin{align*}
\|\wt u\|_{x_0,\lambda} \lesssim \max\big\{1,\big(p(x_0)\big)^{-1/2}\big\}\,(\tau(n) h^{-3d/2} + h^{-d/2+1}\big).
\end{align*}
In addition, an application of the second inequality in Theorem~\ref{thm:Local_embd} implies
\begin{align*}
\sup_{x \in C_{x_0,(1-h)\delta(x_0)}} |\wt u(x)| \lesssim \max\big\{1,\big(p(x_0)\big)^{-1/2}\big\}\,\big(\tau(n)\, h^{-2d} + h^{-d+1}\big).
\end{align*}

Finally, by taking $x=y=x_0$ in the integral form of $\wt K_{x_0}$ in equation~(\ref{Eqn:spectral_rep}), we have the lower bound $\wt K_{x_0}(x_0) \geq c h^{-d}\, \big(p(x_0)\big)^{-1+1/(2\alpha)}$ for some constant $c>0$ that only depends on $\alpha$. 
Therefore, we have the relative error bound
\begin{align*}
\frac{\big|\wt K_\lambda(x_0,x_0) - G(x_0,x_0)\big|}{\big|G(x_0,x_0)\big|} \lesssim  \max\big\{1,\big(p(x_0)\big)^{1/2-1/(2\alpha)}\big\}\sqrt{p(x_0)}\,\big(\tau(n) \, h^{-d} + h\big),
\end{align*}
for any $x_0$ such that the density function satisfies $p(x)\geq p(x_0)/2$ for all $x$ in an $h\log(1/h)$ neighborhood of $x_0$. 
In particular, for any $\alpha \geq 1$, the relative error of estimating the leverage score remains bounded even if the local density $p(x_0)$ tends to zero. 
\end{proof}


\section{MORE ON SIMULATIONS}
\label{sec:exp_details}

In this section, we mainly provide the complete experiment settings and one additional figure to help illustrate our method.
We first describe all the competing methods: original kernel ridge regression; \nystrom methods with uniform sampling (hereinafter referred to as "vanilla"); \nystrom with Recursive-RLS (RC) \citep{musco2017recursive}; \nystrom with BLESS \citep{rudi2018fast}; and \nystrom with spectral analysis (SA, our proposed method).

\subsection{Experiment Settings in Figure~\ref{time_error} in the Main Paper}
\label{exp3}

In this experiment, we compare the runtime and runtime versus error trade-off among Vanilla, RC, BLESS, and our method SA in Figure~\ref{time_error}, under the $3$-d bimodal setting ($\gamma=0.4$) using the \matern kernel ($\nu=1.5$). 
Specifically, the bimodal distribution has two components: with probability $\frac{n}{n+n^\gamma}$ generating a Unif$[0, 1]^3$; and with probability $\frac{n^\gamma}{n+n^\gamma}$ generating a random variable 
with pdf $\prod_{j=1}^3 (5 - 2x_j)$ for $x_j \in [2, 2.5]$, where $n$ is the sample size.

The sample size $n$ ranges from $2,000$ to $500,000$.
In particular, the target function is set as $f^*(x) = g(\|x\|_2/d)$ with 
$g(x) = 1.6 |(x - 0.4) (x - 0.6)| - x (x-1) (x-2) - 0.5$, and i.i.d.~noises follow $\mathcal{N}(0, 0.25)$; 
regularization parameter $\lambda$ is set as $0.075 \cdot n^{-2/3}$, and the bandwidth for Gaussian kernel density estimator is $0.15 n^{-1/7}$. 
The KDE estimator allows a $0.15$ relative error.
The projection dimension for all the methods is set as $5 \cdot n^{1/3}$, while the sub-sampling size $s$ for all the iteration-based \nystrom methods listed is chosen as $1 \cdot n^{1/3}$ due to high time complexity. 
All the results reported in Figure~\ref{time_error} are averaged over 30 replicates.

\subsection{Experiment Settings in Table~\ref{tab:lvg_approx} in the Main Paper}
\label{exp2}

Each method above is run on the \texttt{RadiusQueriesCount} \citep{savva2018explaining, anagnostopoulos2018scalable}(denoted by RQP), 
\texttt{HTRU2} \citep{lyon2016fifty}, 
and \texttt{CCPP} \citep{tufekci2014prediction, tufekci2012local} datasets downloaded from the UCI ML Repository \citep{Dua:2019}. 
Those datasets contain 10000, 17898, and 9568 data points, with 3, 8, and 5 features respectively.
The smoothness parameter of \matern kernel is set as $\nu = 0.5$, and $\alpha \defeq \nu + \frac{d}{2} = \frac{d}{2} + 0.5$.
The regularization parameter $\lambda$ is set as $0.15 \cdot n^{-\frac{2\alpha}{2\alpha+d}}$.
To attain the optimal error rate, the projection dimension of all methods $\floor{2 \cdot n^{\frac{d}{2\alpha+d}}}$;
while the sub-sample size for estimating the statistical leverage scores in RC and BLESS is set as $\floor{1 \cdot n^{\frac{d}{2\alpha+d}}}$.
We still use kernel density estimator to gain density estimation, and the detailed setting of this estimator is almost the same as the last experiment, using Gaussian kernel and the bandwidth $0.5 \cdot n^{-\frac{1}{3}}$. 
All the results reported in Table~\ref{tab:lvg_approx} are averaged over 10 replicates.

\subsection{Experiment Settings in Figure~\ref{lvg} in the Main Paper}

We ran the experiments on the one-dimensional (for the ease of visualization) Unif$[0, 1]$, Beta$(15, 2)$, and a bimodal distribution, as before, with two components: with probability $\frac{n}{n+n^\gamma}$ generating a Unif$[0, 0.5]$; 
and with probability $\frac{n^\gamma}{n+n^\gamma}$ generating a random variable with pdf $(3 - 2x)$ for $x\in[1, 1.5]$, where $n$ is the sample size and $\gamma=0.6$.
In addition, the \matern kernel with smoothness parameter $\nu = 1.5$ is used, and density estimation is performed by a tree-based kernel density estimator.
The number of observations varies from $n=200$ to $10,000$. 
The regularization parameter of the KRR is set as $\lambda = 0.45 \cdot n^{-0.8}$.

A Gaussian kernel is used for density estimation, and the bandwidth is set to $1 \cdot n^{-0.2}$ for Uniform$[0, 1]$ and $0.3 \cdot n^{-1/3}$ for the rest two distributions.
Also, we allow a $0.05$ relative error tolerance for density estimation since highly accurate density estimation is not required for \nystrom methods (cf. Section~\ref{sec:density_estimation}).
While implementing our algorithm, we also apply an ad-hoc modification to avoid the potential instability with a small density value $p(x_i)$, as mentioned in Section~\ref{sec:heuristic} in the main paper.
Particularly, in the case of Beta distribution, if the density of point $x_i$ is smaller than a threshold $h = 0.3 \cdot n^{-0.8}$, 
a weighted average $\frac{0.5h + p(x_i)}{1.5}$ would be used for the subsequent leverage score approximation. 

In Figure~\ref{lvg}, we show our method provides reasonably good approximations to the rescaled leverage scores across all settings. 
In particular, Unif$[0,1]$ is the easiest case (red curves) due to its flat density, which meets Assumption~\ref{a.1} and~\ref{a.2} for almost all design points;
while for points with low density, such as those in the smaller cluster of the bimodal distribution and close to the boundary of Beta$(15, 2)$, the absolute error tends to be large due to the leading constant $C_{x_0}$ in the error bound in Theorem~\ref{thm:main_result}.
Moreover, the relative approximation error has a clear tendency of decreasing as the sample size increases, which is also consistent with our theory.

\subsection{The Additional Experiment for Gaussian Kernels}

To show that our proposed method can also be extended to more kernels other than \matern kernels, 
in this subsection we compare the in-sample prediction error among the methods above in Figure~\ref{gauss_error}, under a dimension-increasing setting ($d=3, 10, 30$ respectively) using a Gaussian kernel with bandwidth $\sigma = 1.5 n^{-\frac{1}{2d+3}}$. 
We still use a bimodal distribution similar to the above one: ($\gamma=0.4$) with probability $\frac{n}{n+n^\gamma}$ generating a Unif$[0, 1]^d$; and with probability $\frac{n^\gamma}{n+n^\gamma}$ generating a random variable 
with pdf $\prod_{j=1}^d (7 - 2x_j)$ for $x_j \in [3, 3.5]$, where $n$ is the sample size.

The sample size $n$ ranges from $1000$ to $100, 000$.
In particular, the target function is set as $f^*(x) = g(\|x\|_2/d) + g(x_1)$ ($x_1$ is the first element of $x$) with 
$g(x) = 1.6 |(x - 0.4) (x - 0.6)| - x (x-1) (x-2) - 0.5$, and i.i.d.~noises follow $\mathcal{N}(0, 0.25)$, which is the same as before; 
regularization parameter $\lambda$ is set as $0.075 \cdot n^{-\frac{d+3}{2d+3}}$, and the bandwidth for the used Gaussian kernel density estimator is tuned for different dimension since when $d$ is large, the density estimation will greatly fluctuate with the size of bandwidth. 
The projection dimension for all the methods is set as $5 \cdot n^{\frac{d}{2d+3}}$, while the sub-sampling size $s$ for all the iteration-based \nystrom  methods listed is chosen as $1 \cdot n^{\frac{d}{2d+3}}$ due to high time complexity. 
All the results reported in Figure~\ref{time_error} are averaged over 20 replicates.
\begin{figure}[h]
\vspace{.3in}
\centerline{\includegraphics[width=\textwidth]{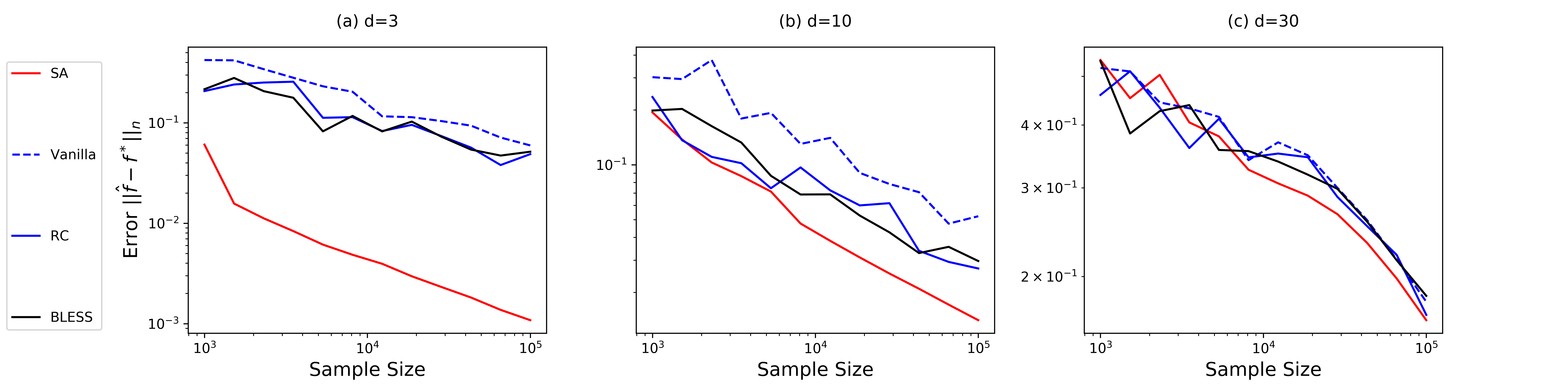}}
\vspace{.3in}
\caption{In-sample prediction error for Gaussian kernels with increasing dimension.}
\label{gauss_error}
\end{figure}
From Figure~\ref{gauss_error}, we observe when $d$ increases, 
all the leverage-based methods will be no longer significantly better than vanilla uniform sampling,
and the in-sample prediction error becomes orders of magnitude larger.
We remark here that an increasing $d$ indeed theoretically violates the assumption of kernel methods on the dimension. 
For the bad performance of KRR, we conjecture that is because in a high dimensional space the input samples get sparser (regarding the Euclidean distance), 
and thus roughly speaking for a certain sample with high density it is also hard to find some points around the sample, which is similar to the case for samples with low density.


\section{APPROXIMATION PROPERTIES}
\label{sec:ek_ppt}

In this section, we prove some useful properties of our equivalent kernel approximation introduced by \matern kernels.
Some parts of the proof rely on the isotropy of the stationary kernels. 
Since the isotropy is a property shared by most common stationary kernels, the proof is expected to be applied to other stationary kernels as well. 
For the reader's convenience, we also prove the corresponding lemmas for Gaussian kernels in Appendix~\ref{Sec:EK_property_Gaussian}.
The proof strategy across the section is as follows, we first focus on one-dimensional cases and utilize the results to prove the conclusion for general multivariate approximation.

\subsection{\matern Kernel}
For simplicity, we ignore some constants (such as $p(x_0)$ that does not change the local shape and scale of $\wt K_\lambda(\cdot,\,x_0)$) and instead consider the rescaled leverage approximation specified by
\begin{align} \label{Eqn:spectral_rep}
\widetilde K_\lambda(x,\,y) = \widetilde K_\lambda(x-y) = \int_{\mb R^d} \frac{e^{2\pi \sqrt{-1} \dotp{s}{x-y}}}{1+\lambda\,(1+\|s\|^2)^{\alpha}}\,ds=\int_{\mb R^d} \frac{\cos(2\pi \dotp{s}{x-y})}{1+\lambda\,(1+\|s\|^2)^{\alpha}}\,ds,
\end{align}
where $\lambda=h^{2\alpha}$.
By the inverse Fourier transform, we have
\begin{align}\label{Eqn:inverse}
f_\lambda(s)=\frac{1}{1+\lambda\,(1+\|s\|^2)^{\alpha}} = \int_{\mb R^d} \widetilde K_\lambda(u)\,e^{-2 \pi \sqrt{-1} \dotp{s}{u}}\,du.
\end{align}

\begin{lemma}\label{Lem:EqKernel_property_matern}
When $2 \alpha = 2 \nu + d \geq d+1$ is an integer, we have:
\begin{enumerate}
\item $\|\widetilde K_\lambda\|_\infty \lesssim h^{-d}$;
\item There exists some constants $C_2>0$ such that
\begin{align*}
|D^{\mathbf j} \wt K_\lambda(x,\,y)|\leq (h^{-|\mathbf j|-d})\,e^{-C_2\,\|x-y\|/h}, \quad |\mathbf j| = 0, 1, \dots, d.
\end{align*}
\label{kernelScale}
\end{enumerate}
\end{lemma}

\begin{proof}

We start with the proof for the univariate case. 
From equation~\eqref{Eqn:spectral_rep}, we have
\begin{align*}
\|\widetilde K_\lambda\|_\infty \leq \int_{-\infty}^\infty \frac{1}{1+\lambda\,(1+s^2)^{\alpha}}\,ds\leq \int_{-\infty}^\infty \frac{1}{1+\lambda s^{2\alpha}}\,ds \lesssim \lambda^{-1/(2\alpha)} = h^{-1},
\end{align*}
which is the first claimed property.
 
To prove the second property, we will apply the residue theorem to the following function 
\begin{align*}
g(z) = \frac{e^{2\pi \sqrt{-1} |u|z}}{1+h^{2\alpha}\,(1+z^2)^{\alpha}}, \quad z\in \mathbb C,
\end{align*}
which is holomorphic on $\mathbb C\setminus\{z_1,\ldots,z_{2\alpha}\}$, where $z_1,\ldots,z_{2\alpha}$ are the $2\alpha$ roots to the equation
\begin{align*}
1+h^{2\alpha}\,(1+z^2)^{\alpha}=0.
\end{align*}
Therefore, $z_{2k-1}$ and $z_{2k}$, for $k=1,\ldots,\alpha$, are the two roots of the equation
\begin{align*}
z^2=h^{-2}\,e^{\sqrt{-1}\frac{2k-1}{\alpha}\pi} -1,
\end{align*}
and $z_{2k-1}=-z_{2k}$. Without loss of generality, we assume $\operatorname{Im}(z_{2k-1}) > 0$.
Direct calculations show that $|\operatorname{Im}(z_{2k-1})| \gtrsim h^{-1}$ and $|z_{2k-1}|\lesssim h^{-1}$ for each $k=1,\ldots,\alpha$.
Now we apply the residue theorem to the following contour integral
\begin{align*}
\int_C g(z) \,dz= \int_C\frac{e^{2\pi \sqrt{-1} |u|z}}{1+h^{2\alpha}\,(1+z^2)^{\alpha}}\,dz,
\end{align*}
where the contour $C$ goes along the real line from $-R$ to $R$ and then counter-clockwise along a semicircle centering at $0$ from $R$ to $-R$, for some sufficiently large constant $R>0$. The residue theorem implies
\begin{align*}
\int_C g(z) \,dz = 2\pi \sqrt{-1} \sum_{k=1}^\alpha \frac{e^{2\pi \sqrt{-1} |u|z_{2k-1}}}{2 \alpha h^{2\alpha} (1+z_{2k-1}^2)^{\alpha-1} z_{2k-1}},
\end{align*}
where we have used the fact that $\{z_{2k-1}\}_{k=1}^\alpha$ are the singularity points inside the contour $C$.
Since $1+h^{2\alpha}\,(1+z_{2k-1}^2)^{\alpha}=0$, the above can be further simplified into 
\begin{align*}
\int_C g(z) \,dz = -\pi \sqrt{-1} \sum_{k=1}^\alpha \frac{e^{2\pi \sqrt{-1} |u|z_{2k-1}}(1+z_{2k-1}^2)}{\alpha z_{2k-1}}.
\end{align*}
Due to the aforementioned properties that $|\operatorname{Im}(z_{2k-1})| \gtrsim h^{-1}$ and $|z_{2k-1}|\lesssim h^{-1}$, we have 
\begin{align*}
\Big|\int_C g(z) \,dz\Big| \lesssim (h + h^{-1})\,e^{-C|u|/h}.
\end{align*}
Finally, we can split the contour $C$ into a straight part (real line) and a curved arc, so that
\begin{align*}
\int_C g(z) \,dz=\int_{(-R,R)} g(z) \,dz + \int_{\text{arc}} g(z) \,dz,
\end{align*}
where the arc part satisfies
\begin{align*}
\Big|\int_{\text{arc}} g(z) \,dz\Big| \leq \pi R\cdot\sup_{\text{arc}}\Big|\frac{e^{2\pi \sqrt{-1} |u|z}}{1+h^{2\alpha}(1+z^2)^\alpha}\Big|\leq \frac{\pi R}{h^{2\alpha} (R^{2}-1)^{\alpha}-1}.
\end{align*}
By taking $R\to\infty$ (note that $\alpha>1/2$) and putting all pieces together, we finally reach
\begin{align*}
|\widetilde K_\lambda(x-y)|=\Big|\int_{-\infty}^\infty \frac{e^{2\pi \sqrt{-1} s(x-y)}}{1+\lambda\,(1+s^2)^{\alpha}}\,ds\Big|  \lesssim (h+h^{-1})\,e^{-C|x-y|/h},
\end{align*}
which is part of the second desired property.

To complete the proof of the second property, we still need to bound the derivative of the equivalent kernel. 
Recall the differentiation property of Fourier transform, and $\ms F[\widetilde K'_\lambda]$ could be written as:
\begin{align*}
\ms F[\widetilde K'_\lambda] = \frac{2 \pi \sqrt{-1} s}{1 + h^{2\alpha}\,(1+s^2)^{\alpha}}
\end{align*}
Following a similar way, we can accordingly reset function $g$ as:
\begin{align*}
g(z) = \frac{e^{2\pi \sqrt{-1} |u| z} 2 \pi \sqrt{-1} z}{1+h^{2\alpha}\,(1+z^2)^{\alpha}}, \quad z\in \mathbb C,
\end{align*}
and by the same procedure obtain the following equality
\begin{align*}
\int_C g(z) \,dz = 4 \pi^2 \sum_{k=1}^\alpha \frac{e^{2 \pi \sqrt{-1} |u| z_{2k-1}}(1+z_{2k-1}^2) z_{2k-1}}{2 \alpha z_{2k-1}} 
= 2 \frac{\pi^2}{\alpha} \sum_{k=1}^\alpha e^{2\pi \sqrt{-1} |u|z_{2k-1}}(1+z_{2k-1}^2).
\end{align*}


As for the integral over the arc part, its value is still negligible due to a finer control.
Note over the arc, $z = R \cos(\theta) + \sqrt{-1} R \sin(\theta), \theta \in [0, \pi]$, 
and
\begin{align*}
&\Big|\int_{\text{arc}} g(z) \,dz\Big| \\
&\qquad = \Big| 2 \pi \sqrt{-1} \int_0^\pi e^{-2 \pi |u| R \sin(\theta)} \frac{e^{2 \pi \sqrt{-1} |u| R \cos(\theta)} z(\theta)}{1+h^{2\alpha}\,(1+z^2(\theta))^{\alpha}} (-R \sin(\theta) + \sqrt{-1} R \cos(\theta)) \dd \theta \Big| \\
&\qquad \leq \frac{2 \pi^2 R^2}{h^{2 \alpha} (R^{2} - 1)^{\alpha} - 1} 
    \cdot 2 \int_0^{\frac{\pi}{2}} e^{-2 \pi |u| R \sin(\theta)} \dd \theta.
\end{align*}
To bound the rest integral, we utilize the fact that $\sin(\theta) / \theta$ is decreasing in $(0, \pi/2]$, and $\sin(\theta) \geq \frac{2}{\pi} \theta, \forall \theta \in (0, \pi/2]$.
Therefore,
\begin{align*}
\Big|\int_{\text{arc}} g(z) \,dz\Big| 
\leq \frac{2 \pi^2 R^2}{h^{2 \alpha} (R^{2} - 1)^{\alpha} - 1} 
    \cdot 2 \int_0^{\frac{\pi}{2}} e^{-4 |u| R \theta} \dd \theta 
\leq \frac{C}{|u|} \frac{\pi^2 R}{h^{2 \alpha} (R^{2} - 1)^{\alpha} - 1}.
\end{align*}
By taking $R \to \infty$ (note that $2 \alpha > d = 1$ here), the magnitude of the integral over the arc would vanish. 

Putting all pieces together, we finally reach
\begin{align*}
|\widetilde K_\lambda'(x-y)|=\Big|\int_{-\infty}^\infty \frac{e^{2\pi \sqrt{-1} s (x-y)} (2 \pi \sqrt{-1} s)}{1+\lambda\,(1+s^2)^{\alpha}}\,ds\Big|  \lesssim (1+h^{-2})\,e^{-C|x-y|/h},
\end{align*}
which complete the proof of the second property for univariate cases.

The proof for the multivariate claim will utilize the univariate conclusion before.
By using the polar coordinate transform in Appendix~\ref{sec:polar}, we can reduce the original multivariate integral to a univariate one (cf. equation~(\ref{eqn:mv_integral})).
The rescaled leverage $\wt K_\lambda(0)$ would be proportional to:
\begin{align*}
\int_0^\infty \frac{r^{d-1}}{1 + \lambda(1+r^2)^\alpha} \dd r.
\end{align*}
which is of the scale $h^{-d}$ by using the same technique as before.
Therefore the first claim in this lemma has been proved.


For the second claim, we would heavily utilize the isotropy trick to simplify the proof.
By the isotropy of \matern kernels and our $\wt K_\lambda(u)$, we only need to consider a special input $\wt u = (\|u\|, 0, \dots, 0)$.
That is motivated by the observation that we can always do the coordinate transformation $s = T \cdot t$, where $T$ is an orthogonal matrix and its first row $T_{1,\cdot} = u / \|u\|$.
In that case, the original mixed derivative $D^{\mathbf j} \wt K_\lambda(u)$ could be expressed as 
\begin{align*}
\int_{\mb R^d} \frac{e^{2\pi \sqrt{-1} \dotp{u}{s}}}{1+h^{2\alpha}\,(1 + \|s\|^2)^{\alpha}} \prod_{i=1}^d (2\pi \sqrt{-1} s_i)^{\mathbf j_i} \dd s
= \int_{\mb R^d} \frac{e^{2\pi \sqrt{-1} \|u\| t_1}}{1+h^{2\alpha}\,(1 + \|t\|^2)^{\alpha}} \prod_{i=1}^d (2\pi \sqrt{-1} \dotp{T_{i,\cdot}}{t})^{\mathbf j_i} \dd t,
\end{align*}
which is of the same scale as $\max_{|\mathbf j'|=|\mathbf j|} |D^{\mathbf j'} \wt K_\lambda(\wt u)|$.
Under the settings above, the target considered would be reduced to
\begin{align*}
& \big| \int_{\mb R^d} \frac{e^{2\pi \sqrt{-1} \|u\| s_1}}{1+h^{2\alpha}\,(1 + \sum_{i=1}^{d-1} s_i^2 + s_d^2)^{\alpha}} \prod_{i=1}^d (s_i)^{\mathbf j_i} \dd s \big| \\
=& \big| \int_{\mb R^{d-1}} \prod_{i=2}^d (s_i)^{\mathbf j_i} \int_{-\infty}^\infty \frac{e^{2\pi \sqrt{-1} \|u\| s_1} s_1^{\mathbf j_1}}{1+h^{2\alpha}\,(1 + \|s_{-1}\|^2 + s_1^2)^{\alpha}}  \dd s_1 \dd s_{-1} \big| \\
\leq& \int_{\mb R^{d-1}} \prod_{i=2}^d |s_i|^{\mathbf j_i} \big| \int_{-\infty}^\infty \frac{e^{2 \pi \sqrt{-1} \|u\| s_1} s_1^{\mathbf j_1}}{1+h^{2\alpha}\,(1 + \|s_{-1}\|^2 + s_1^2)^{\alpha}} \dd s_1 \big| \dd s_{-1}.
\end{align*}

The next important step is to take the expression $1 + \|s_{-1}\|^2$ as a constant, and again apply the residue theorem to bound the internal integral as
\begin{align*}
& \big| \frac{\pi \sqrt{-1}}{\alpha} \sum_{k=1}^\alpha e^{2 \pi \sqrt{-1} \|u\| z_{2k-1}}(1 + z_{2k-1}^2) (z_{2k-1})^{\mathbf j_1 - 1} \big| \\
=& \big| \frac{\pi}{\alpha} \sum_{k=1}^\alpha e^{2 \pi \sqrt{-1} \|u\| z_{2k-1}} h^{-2} e^{\sqrt{-1} \frac{2k-1}{\alpha}\pi} (z_{2k-1})^{\mathbf j_1 - 1} \big|,
\end{align*}
where $z_{2k-1} \defeq a_k + \sqrt{-1} \cdot b_k$, and if we denote $\theta_k \defeq \frac{2k-1}{\alpha}\pi$, 
\begin{align*}
a_k^2 + b_k^2 &= |z_{2k-1}^2| = (h^{-4} \sin^2(\theta_k) + (h^{-2} \cos(\theta_k) - \|s_{-1}\|^2 - 1)^2)^{\frac 12}, \\
2 b_k^2 &= (a_k^2 + b_k^2) - (a_k^2 - b_k^2) \\
&= (h^{-4} \sin^2(\theta_k) + (h^{-2} \cos(\theta_k) - \|s_{-1}\|^2 - 1)^2)^{\frac 12} - (h^{-2} \cos(\theta_k) - \|s_{-1}\|^2 - 1) \\
&= \frac{h^{-4} \sin^2(\theta_k)}{(h^{-4} \sin^2(\theta_k) + (h^{-2} \cos(\theta_k) - \|s_{-1}\|^2 - 1)^2)^{\frac 12} + (h^{-2} \cos(\theta_k) - \|s_{-1}\|^2 - 1)}
\end{align*}
Note this time the magnitude of $\|s_{-1}\|$ matters a lot, and we need to divide the outside integral into two domains, $D_1 \defeq \{|h^{-2} \cos(\theta_k) - \|s_{-1}\|^2 - 1| \leq 3 h^{-2} |\cos(\theta_k)|\}$ and $D_2 \defeq \{|h^{-2} \cos(\theta_k) - \|s_{-1}\|^2 - 1| > 3 h^{-2} |\cos(\theta_k)|\}$.

Over domain $D_1$, we can imply $\|s_{-1}\|^2 \leq 4 h^{-2}$, and similar to the univariate case we can have $b_k \gtrsim h^{-1}$ and $|z_{2k-1}| = \sqrt{a_k^2 + b_k^2} = \Theta(h^{-1})$.
The corresponding integral would be bounded by a constant multiple of
\begin{align*}
& \int_{D_1} \prod_{i=2}^d |s_i|^{\mathbf j_i} h^{-2} e^{-C \|u\| b_k} (z_{2k-1})^{\mathbf j_1 - 1} \dd s_{-1} \\
\lesssim& \int_{\|s_{-1}\|^2 \leq 4 h^{-2}} \|s_{-1}\|^{|\mathbf j|-\mathbf j_1} h^{-2} e^{-C \|u\| h^{-1}} h^{-(\mathbf j_1 - 1)} \dd s_{-1} \\
\lesssim& e^{-C \|u\| h^{-1}} h^{-(\mathbf j_1 + 1)} \int_0^{2 h^{-1}} r^{|\mathbf j| - \mathbf j_1 + d - 2} \dd r
\lesssim e^{-C \|u\| h^{-1}} h^{-|\mathbf j| - d}.
\end{align*}

For domain $D_2$, we could notice $b_k$ would be much bigger than in $D_1$ as now $\|s_{-1}\|$ tends to dominate $h^{-1}$.
Specifically, $2 b_k^2 \geq 1+\|s_{-1}\|^2 + h^{-2} (1-\cos(\theta_k))$ and thus $b_k \geq C (\|s_{-1}\| + h^{-1})$. Considering $|z_{2k-1}| = \Theta(\|s_{-1}\|)$, we have
\begin{align*}
& \int_{D_2} \prod_{i=2}^d |s_i|^{\mathbf j_i} h^{-2} e^{-C \|u\| b_k} (z_{2k-1})^{\mathbf j_1 - 1} \dd s_{-1} \\
\lesssim& h^{-2} e^{-C \|u\| h^{-1}} \int_{\|s_{-1}\|^2 > c^2 h^{-2}} \|s_{-1}\|^{|\mathbf j|-\mathbf j_1} e^{-C \|u\| \|s_{-1}\|} \|s_{-1}\|^{\mathbf j_1 - 1} \dd s_{-1} \\
\lesssim& h^{-2} e^{-C \|u\| h^{-1}} \int_{c h^{-1}}^\infty r^{|\mathbf j| + d - 3} e^{-C' \|u\| r} \dd r
\lesssim h^{-2} e^{-C \|u\| h^{-1}} h^{- |\mathbf j| - d + 3} e^{-C' \|u\| h^{-1}} \\
\lesssim& e^{-C \|u\| h^{-1}} h^{- |\mathbf j| - d + 1} 
\lesssim e^{-C \|u\| h^{-1}} h^{- |\mathbf j| - d}.
\end{align*}

Combining the two pieces, we finally have
\begin{align*}
|D^{\mathbf j} \wt K_\lambda(u)| &\lesssim e^{-C \|u\| h^{-1}} h^{- |\mathbf j| - d}.
\end{align*}

\end{proof}

\subsection{Gaussian Kernel}
\label{Sec:EK_property_Gaussian}

Similarly, we specify the equivalent kernel as
\begin{align} \label{Eqn:spectral_rep_exp}
\widetilde K_\lambda(x,\,y) = \int_{\mb R^d} \frac{e^{2\pi \sqrt{-1} \dotp{s}{x-y}}}{1 + \frac{\lambda}{\sigma^d} \, e^{\|\sigma s\|^2}}\,ds = \int_{\mb R^d} \frac{\cos(2\pi \dotp{s}{x-y})}{1 + \frac{\lambda}{\sigma^d} \, e^{\|\sigma s\|^2}}\,ds,
\end{align}
where the bandwidth $\sigma$ is introduced due to its importance to exponential kernels. 
The scale of the bandwidth is set as $\m O(\lambda^{\frac{1}{2 \alpha}})$, where $\alpha$ is the corresponding parameter of the most suitable \matern kernel attaining the optimal error rate in kernel ridge regression problems.
We also define an auxiliary parameter $h$ as $ h^{-2} \equiv \ln{\frac{\sigma^d}{\lambda}}$ for simplicity of notation.
In practice, $\sigma$ would be specified in a way similar to \matern case;
a parameter $\alpha > d/2$ would be first chosen, and then let $\sigma = O(\lambda^{\frac{1}{2 \alpha}}) \to 0$, which implies $h$ here has the magnitude $O(\log^{-\frac{1}{2}}(n))$.
In the following lemma, we will again start from a univariate case.

\begin{lemma}\label{Lem:EqKernel_property_gaussian}
	Given the equivalent kernel introduced above, we have ($\wt {\m O}(\cdot)$ means $\m O(\cdot)$ modulo poly-log terms):
	\begin{enumerate}
		\item $\|\widetilde K_\lambda\|_\infty \lesssim \sigma^{-d} h^{-d} = \wt {\m O} (\sigma^{-d})$;
        \item There exists some constant $C_3>0$ such that for $|\mathbf j| \leq d$,
        \begin{align*}
        |D^{\mathbf j} \wt K_\lambda(x,\,y)| \lesssim (\sigma h)^{-|\mathbf j|-d} \, e^{-C_3 \,|x-y| \sigma^{-1} h} = \wt {\m O} (\sigma^{-|\mathbf j|-d} e^{-C_3 \,|x-y| \sigma^{-1} h}).
        \end{align*}
		\label{kernelScale_exp}
	\end{enumerate}
\end{lemma}
\begin{proof}
	
Again we begin with univariate cases. From equation~\eqref{Eqn:spectral_rep_exp}, we have
\begin{align*}
\|\widetilde K_\lambda\|_\infty 
\leq \int_{-\infty}^\infty \frac{1}{1 + \frac{\lambda}{\sigma} \, e^{(\sigma s)^2}}\,ds 
\leq 2 \Big[\int_{0}^{\sigma^{-1} h^{-1}} \frac{1}{1+0} \,ds + 
    \int_{\sigma^{-1} h^{-1}}^{\infty} \frac{1}{\frac{\lambda}{\sigma} \, e^{(\sigma s)^2}} \,ds \Big]
\end{align*}
The integral is divided into two parts in which $1$ and $\frac{\lambda}{\sigma}\,e^{({\sigma}{s})^2}$ dominate respectively.
What's more, applying the property of error function that $\int_x^{\infty} e^{-t^2} \,dt = \m O(\frac{e^{-x^2}}{x})$ (it holds when $x$ is large enough), we can further obtain
\begin{align*}
\|\widetilde K_\lambda\|_\infty  
\lesssim 2 \Big[\sigma^{-1} h^{-1} + \frac{1}{\lambda} h e^{-h^{-2}} \Big]
\lesssim \sigma^{-1} h^{-1}
\end{align*}
which is the first claimed property.

To prove the second one, we still apply the residue theorem to the following function ($|x-y|$ is denoted as $|u|$ for simplicity),
\begin{align*}
g(z) = \frac{e^{2\pi \sqrt{-1} |u| z}}{1 + \frac{\lambda}{\sigma} \, e^{(\sigma z)^2}}, \quad z\in \mathbb C,
\end{align*}
which is holomorphic on the complex plane except the roots $z_i, i \in \mathbb Z$ to the following equation:
\begin{align*}
1 + \frac{\lambda}{\sigma} \, e^{(\sigma z)^2} = 0
\end{align*}
Therefore, $z_{2k-1}$ and $z_{2k}$, for $k \in \mathbb Z$, are the two roots of the equation:
\begin{align*}
\sigma^2 z^2 =  (h^{-2} + \sqrt{-1} (2k-1) \pi)
\end{align*}
and $z_{2k-1} = -z_{2k}$. Without loss of generality, we assume $\operatorname{Im}(z_{2k-1}) > 0$.
Direct calculations roughly show that $|\operatorname{Im}(z_{2k-1})| \gtrsim \sigma^{-1}$ and $|z_{2k-1}| \lesssim \sigma^{-1}$ for each $k \in \mathbb Z$. Further analysis would be provided later.

Note we could only focus on the case $|u| > \sigma h^{-1}$, since otherwise 
\begin{align*}
|\widetilde K_\lambda(u)| \leq C (\sigma h)^{-1} \leq C \,e^{C_3} (\sigma h)^{-1} e^{-C_3 |u| \sigma^{-1} h},
\end{align*} 
and the claimed property would be proved automatically.
Now we apply the residue theorem to the following contour integral
\begin{align*}
\int_C g(z) \,dz= \int_C\frac{e^{2\pi \sqrt{-1} |u|z}}{1 + \frac{\lambda}{\sigma} \, e^{(\sigma z)^2}} \,dz,
\end{align*}
where the contour $C$ goes along the real line from $-R$ to $R$ and then counter-clockwise along a semicircle centering at $0$ from $R$ to $-R$, for some sufficiently large constant $R>0$. 
Denote the index set $A_R$ as the set of all the indices $k$ that the roots $\{z_{2k-1}\}_k$ are inside the contour $C$. The residue theorem implies
\begin{align*}
\int_C g(z) \,dz = 2\pi \sqrt{-1} \sum_{k \in A_R} \frac{e^{2\pi \sqrt{-1} |u| z_{2k-1}}}{2 \sigma \lambda e^{(\sigma z_{2k-1})^2} z_{2k-1}},
\end{align*}
Since $1 + \frac{\lambda}{\sigma} \, e^{(\sigma z)^2} = 0$, the above expression can be further simplified into 
\begin{align*}
\int_C g(z) \,dz = -\pi \sqrt{-1} \sum_{k \in A_R} \frac{e^{2 \pi \sqrt{-1} |u| z_{2k-1}}}{\sigma^2 z_{2k-1}}.
\end{align*}
Note the set $A_R$ is symmetric about $0$ and goes to $\mathbb Z$ as $R \to \infty$. We can first pair the opposite $k$ and denote $z_{2k-1} = a_k + b_k \sqrt{-1}, z_{1-2k} = -a_k + b_k \sqrt{-1}$ for convenience. In this case,
\begin{align*}
& \frac{e^{2\pi \sqrt{-1} |u| z_{2k-1}}}{\sigma^2 z_{2k-1}} + 
\frac{e^{2\pi \sqrt{-1} |u| z_{1-2k}}}{\sigma^2 z_{1-2k}} = 
\frac{e^{-2\pi |u| b_k}}{-\sigma^2 (a_k^2 + b_k^2)} 2 \sqrt{-1} (b_k \cos(2\pi|u|a_k) - a_k \sin(2\pi|u|a_k)) \\
=& \frac{2 \sqrt{-1} e^{-2\pi |u| b_k}}{-\sigma^2 \sqrt{a_k^2 + b_k^2}} \cos(2\pi|u|a_k 
+ \arctan(a_k/b_k))
\end{align*}
And hence the sequence of the integral over the semicircle $C$ with radius $R$ would converge to
\begin{align*}
\frac{-2 \pi}{\sigma^2} \sum_{k=1}^\infty \frac{e^{-2 \pi |u| b_k}}{\sqrt{a_k^2 + b_k^2}} \cos(2 \pi |u| a_k + \arctan(a_k/b_k))
\leq \frac{2 \pi}{\sigma^2} \sum_{k=1}^\infty \frac{e^{-2 \pi |u| b_k}}{\sqrt{a_k^2 + b_k^2}}	
\end{align*}


To further analyze the scale of the infinite series, we need to uncover the form of the coefficient $a_k, b_k$.
Recall $z_{2k-1}^2 = \frac{1}{\sigma^2} (\ln \frac{\sigma}{\lambda} + (2k-1) \pi)$, and the corresponding derivation is,
\begin{align*}
a_k^2 + b_k^2 &= \sigma^{-2} (\ln^2 \frac{\sigma}{\lambda} + ((2k-1) \pi)^2)^\frac{1}{2} \\
2 b_k^2 &= (a_k^2 + b_k^2) - (a_k^2 - b_k^2) = \sigma^{-2} \Big[ (\ln^2 \frac{\sigma}{\lambda} + ((2k-1) \pi)^2)^\frac{1}{2} - (\ln \frac{\sigma}{\lambda}) \Big]\\
&= \frac{\sigma^{-2} ((2k-1) \pi)^2}{(\ln^2 \frac{\sigma}{\lambda} + ((2k-1) \pi)^2)^\frac{1}{2} + (\ln \frac{\sigma}{\lambda})}
\end{align*}
Denote $H \equiv (\frac {h^{-2}}{\pi} + 1) / 2$. The curve of $a_k, b_k$ could be roughly divided into two stages, $k \leq \lf H \rf$ and $k \geq \lc H \rc$. In the first stage, $a_k^2 + b_k^2 \geq \sigma^{-2} \ln \frac{\sigma}{\lambda} = (\sigma h)^{-2}$ and $2 b_k^2 \geq \sigma^{-2} \frac{((2k-1) \pi)^2}{2 \ln \frac{\sigma}{\lambda}} = \sigma^{-2} \frac{((2k-1) \pi)^2}{2 h^{-2}}$; 
in the second stage, $a_k^2 + b_k^2 \geq \sqrt{2} \pi \sigma^{-2} (2k-1)$ and $2 b_k^2 \geq \sigma^{-2} \frac{((2k-1) \pi)^2}{3 \pi (2k-1)} = \sigma^{-2} \frac{(2k-1) \pi}{3}$.
The infinite series could be therefore bounded as
\begin{align*}
\frac{2 \pi}{\sigma^2} \sum_{k=1}^{\lf H \rf} \frac{e^{-2 \pi |u| b_k}}{\sqrt{a_k^2 + b_k^2}} &\lesssim \sigma^{-1} h \sum_{k=1}^{\lf H \rf} {e^{-C_3 |u| \sigma^{-1} h k}} \\
\frac{2 \pi}{\sigma^2} \sum_{k=\lc H \rc}^{\infty} \frac{e^{-2 \pi |u| b_k}}{\sqrt{a_k^2 + b_k^2}} &\lesssim \sigma^{-2} \sum_{k=\lc H \rc}^{\infty} \frac{e^{-C_3 |u| \sigma^{-1} \sqrt{k}}}{\sigma^{-1} \sqrt{k}}
\end{align*}
The two series above will converge rapidly when $n$ is large enough.
The scale of the two series above will therefore depend on their own first terms. 
The first series would be bounded as a constant multiple of $\sigma^{-1} h \cdot h^{-2} e^{-C_3 |u| \sigma^{-1} h \cdot 1} = (\sigma h)^{-1} e^{-C_3 |u| \sigma^{-1} h}$; due to the decreasing sequence, the last series could be bounded in the following way (note $|u| > \sigma h^{-1}$),    
\begin{align*}
\sigma^{-1} \sum_{k=\lc H \rc}^{\infty} \frac{e^{-C_3 |u| \sigma^{-1} \sqrt{k}}}{\sqrt{k}} 
&\leq \sigma^{-1} \Big[ \frac{e^{-C_3 |u| \sigma^{-1} \sqrt{\lc H \rc}}}{\sqrt{\lc H \rc}} + \int_{\lc H \rc}^{\infty} \frac{e^{-C_3 |u| \sigma^{-1} \sqrt{x}}}{\sqrt{x}} d\,x \Big] \\
&\leq \sigma^{-1} \Big[ h e^{-C_3 |u| (\sigma h)^{-1}} + \frac{2 \sigma}{C_3 |u|} e^{-C_3 |u| (\sigma h)^{-1}} \Big] \lesssim \sigma^{-1} h e^{-C_3 |u| (\sigma h)^{-1}}
\end{align*}
And hence the whole series would be bounded by $(\sigma h)^{-1} e^{-C_3 |u| \sigma^{-1} h}$.

Then, we split the contour $C$ into a straight part (real line) and a curved arc, so that
\begin{align*}
\int_C g(z) \,dz=\int_{(-R,R)} g(z) \,dz + \int_{\text{arc}} g(z) \,dz,
\end{align*}
where the arc part satisfies $z = R e^{\sqrt{-1} \theta}, \theta \in [0, \pi]$ and hence,
\begin{align*}
\Big|\int_{\text{arc}} g(z) \,dz\Big| = \Big|\int_0^\pi \frac{e^{2 \pi \sqrt{-1} |u| R e^{\sqrt{-1} \theta}}}{1 + \frac{\lambda}{\sigma} \, e^{\sigma^2 R^2 e^{2 \sqrt{-1} \theta}}} \sqrt{-1} R e^{\sqrt{-1} \theta} \dd \theta\Big|
\end{align*}
where the module of the integrand could be bounded by $\frac{e^{-2 \pi |u| R \sin \theta} R}{|1 - \frac{\lambda}{\sigma} \, e^{\sigma^2 R^2 \cos (2 \theta)}|}$.
By taking $R \to \infty$ and requiring $|u| > 0$, we could observe that when $\sin \theta$ is bounded away from $0$ the integrand is exponentially decaying;
when $\sin \theta$ is nearly zero $\cos{2 \theta} = 1 - 2 \sin^2 \theta \gg 0$ and the integrand would also go to $0$. 
That's to say, the whole integral $\int_{\text{arc}} g(z) \,dz \to 0$.
Putting all pieces together, we finally reach
\begin{align*}
|\widetilde K_\lambda(x-y)| = \int_{-\infty}^\infty \frac{e^{2\pi \sqrt{-1} s(x-y)}}{1 + \frac{\lambda}{\sigma} \, e^{(\sigma s)^2}}\,ds  
\lesssim \sigma^{-1} h e^{-C_3 |u| \sigma^{-1} h} + \sigma^{-1} h^{-1} e^{-C_3 |u| \sigma^{-1} h^{-1}},
\end{align*}
which is part of the second desired property and similar to the conclusion in (\ref{Lem:EqKernel_property_matern}).

To complete the proof of the second property, we still need to bound the derivative of the equivalent kernel. 
Recall the differentiation property of Fourier transform, and $\ms F[\widetilde K'_\lambda]$ could be written as:
\begin{align*}
\ms F[\widetilde K'_\lambda] = \frac{2 \pi \sqrt{-1} s}{1 + \frac{\lambda}{\sigma} \, e^{(\sigma s)^2}}.
\end{align*}
With the expression above we can bound the sup norm of the derivative, 
\begin{align*}
\|\widetilde K_\lambda'\|_\infty 
\leq 4 \pi \int_{0}^\infty \frac{s}{1 + \frac{\lambda}{\sigma} \, e^{(\sigma s)^2}}\,ds 
\leq 4 \pi \Big[\int_{0}^{\sigma^{-1} h^{-1}} \frac{s}{1+0} \,ds + 
\int_{\sigma^{-1} h^{-1}}^{\infty} \frac{s}{\frac{\lambda}{\sigma} \, e^{(\sigma s)^2}} \,ds \Big].
\end{align*}
The integral is divided into two parts in which $1$ and $\frac{\lambda}{\sigma}\,e^{(\frac{s}{\sigma})^2}$ dominate respectively.
We can further obtain
\begin{align*}
\|\widetilde K_\lambda'\|_\infty  
\lesssim \Big[(\sigma h)^{-2} + \frac{1}{2 \lambda \sigma} e^{-h^{-2}} \Big]
\lesssim (\sigma h)^{-2}
\end{align*}

To exactly analyze behavior of the derivative of equivalent kernels, we can accordingly reset function $g$ as:
\begin{align*}
g(z) = \frac{e^{2 \pi \sqrt{-1} |u| z} 2 \pi \sqrt{-1} z}{1 + \frac{\lambda}{\sigma} \, e^{(\sigma z)^2}}, \quad z\in \mathbb C,
\end{align*}
and by the same procedure obtain the following inequality:
\begin{align*}
\int_C g(z) \,dz &= 2 \pi^2 \sqrt{-1} \sum_{k \in A_R} \frac{e^{2 \pi \sqrt{-1} |u| z_{2k-1}}}{\sigma^2} \\
&= \frac{2 \pi^2}{\sigma^2} \sum_{k=1}^{\infty} (e^{2 \pi \sqrt{-1} |u| (a_k + b_k \sqrt{-1})} + e^{2 \pi \sqrt{-1} |u| (-a_k + b_k \sqrt{-1})}) \\
&= \frac{4 \pi^2}{\sigma^2} \sum_{k=1}^{\infty} e^{-2 \pi |u| b_k} (\cos(2 \pi |u| a_k))
\leq \frac{4 \pi^2}{\sigma^2} \sum_{k=1}^{\infty} e^{-2 \pi |u| b_k}.
\end{align*}
We would only focus on the case $|u| > (\sigma / h)^{-1}$, 
since otherwise $|\widetilde K'_\lambda(u)| \leq C (\sigma h)^{-2}$, which is bounded by $\leq C \,e^{C_3} (\sigma h)^{-2} e^{-C_3 |u| \sigma^{-1} h}$, 
and the claimed property would be proved automatically.	
Due to the aforementioned division of the series, we have (note $|u| > (\sigma / h)^{-1}$),
\begin{align*}
\int_C g(z) \,dz &\leq \frac{4 \pi^2}{\sigma^2} (\sum_{k=1}^{\lf H \rf} e^{-2 \pi |u| b_k} + \sum_{k = \lc H \rc}^{\infty} e^{-2 \pi |u| b_k}) \\
&\lesssim \frac{4 \pi^2}{\sigma^2} \Big[ h^{-2} e^{-C_3 |u| \sigma^{-1} h} + e^{-C_3 |u| \sigma^{-1} h^{-1}} +  
\int_{k = \lc H \rc}^{\infty} e^{-C_3 |u| \sigma^{-1} \sqrt{k}} d\,k \Big] \\
& \lesssim \frac{1}{(\sigma h)^2} \Big[ e^{-C_3 |u| \sigma^{-1} h} + 
\frac{2 \sigma}{C_3 |u|} \Big( h^{-2} e^{-C_3 |u| \sigma^{-1} h^{-1}}
+ \frac{\sigma}{C_3 |u|} e^{-C_3 |u| \sigma^{-1} h^{-1}} \Big) \Big] \\
& \lesssim \frac{1}{(\sigma h)^2} e^{-C_3 |u| \sigma^{-1} h}.
\end{align*}
As for the integral over the arc part, its value is still negligible as
\begin{align*}
\Big|\int_{\text{arc}} g(z) \,dz\Big| \leq \Big|\int_0^\pi \frac{e^{2 \pi \sqrt{-1} |u| R e^{\sqrt{-1} \theta}} 2 \pi \sqrt{-1} R e^{\sqrt{-1} \theta}}{1 + \frac{\lambda}{\sigma} \, e^{\sigma^2 R^2 e^{2 \sqrt{-1} \theta}}} R \,d \theta\Big| 
\end{align*}
where the module of the integrand could be bounded by $\frac{e^{-2 \pi |u| R \sin \theta} 2 \pi R^2}{|1 - \frac{\lambda}{\sigma} \, e^{\sigma^2 R^2 \cos (2 \theta)}|}$. The bound goes to $0$ when $R \to \infty$ and $|u| > 0$ are assumed as before, and the integral over the arc is again negligible.
Putting all pieces together, we finally reach
\begin{align*}
|\widetilde K'_\lambda(x-y)| 
= \Big|\int_{-\infty}^\infty \frac{e^{2 \pi \sqrt{-1} s (x-y)} (2 \pi \sqrt{-1} s)}{1 + \frac{\lambda}{\sigma} \, e^{(\sigma s)^2}}\,ds\Big|  
\lesssim (\sigma h)^{-2} \, e^{-C_3 |u| \sigma^{-1} h},
\end{align*}
which completes the proof of the second property in univariate cases.

For multivariate cases, again by applying polar coordinate transformation, 
the original integral could be bounded as
\begin{align*}
\|\wt K_\lambda\|_\infty &\lesssim \int_{0}^\infty \frac{r^{d-1}}{1 + \frac{\lambda}{\sigma^d} \, e^{(\sigma r)^2}}\,\dd r 
\leq \int_{0}^{\sigma^{-1} h^{-1}} \frac{r^{d-1}}{1+0} \dd r + \int_{\sigma^{-1} h^{-1}}^{\infty} \frac{r^{d-1}}{\frac{\lambda}{\sigma^d} e^{(\sigma r)^2}} \,dr \\
&\lesssim (\sigma h)^{-d} + \frac{1}{\lambda} \int_{h^{-1}}^{\infty} r^{d-1} e^{-r^{2}} \dd r
\lesssim (\sigma h)^{-d} - \frac{1}{\lambda} \int_{h^{-1}}^{\infty} r^{d-2} \dd e^{-r^{2}}.
\end{align*}
After repeatedly using integration by parts, the last term could be further bounded by $(\sigma h)^{-d} + \sigma^{-d} h^{-(d-2)}$,
which validates the first claim in this lemma.


For the second claim, we would still use the same strategy, utilizing the isotropy and some other tricks, as in the proof for \matern kernels. 
Specifically, we would focus on the special case
\begin{align*}
\big| \int_{\mb R^d} \frac{e^{2\pi \sqrt{-1} \|u\| s_1}}{1 + \frac{\lambda}{\sigma^d}\,e^{\sigma^2 (\|s_{-1}\|^2 + s_1^2)}} \prod_{i=1}^d s_i^{\mathbf j_i} \dd s \big|
\leq \int_{\mb R^{d-1}} \prod_{i=2}^d |s_i|^{\mathbf j_i} \cdot \big| \int_{-\infty}^\infty \frac{e^{2 \pi \sqrt{-1} \|u\| s_1} |s_1|^{\mathbf j_1}}{1 + \frac{\lambda}{\sigma^d}\,e^{\sigma^2 (\|s_{-1}\|^2 + s_1^2)}} \dd s_1 \big| \dd s_{-1},
\end{align*}
and define $h^{-2} \defeq \ln(\frac{\sigma^d}{\lambda})$, $t_s \defeq |h^{-2} - \sigma^2 \|s_{-1}\|^2|$, $\lambda \lesssim \sigma^d \lesssim h^{-d}$.
Again we divide the integral into two different domains, $D_1 \defeq \{\sigma^2 \|s_{-1}\|^2 < h^{-2}\}$ and $D_2 \defeq \{\sigma^2 \|s_{-1}\|^2 \geq h^{-2}\}$. 
Moreover, we apply residue theorem to the internal integral and similarly have
\begin{align*}
\big| \int_{-\infty}^\infty \frac{e^{2 \pi \sqrt{-1} \|u\| s_1} |s_1|^{\mathbf j_1}}{1 + \frac{\lambda}{\sigma^d}\,e^{\sigma^2 (\|s_{-1}\|^2 + s_1^2)}} \dd s_1 \big|
&= \big| 2 \pi \sqrt{-1} \sum_{k = -\infty}^{\infty} \frac{e^{2 \pi \sqrt{-1} \|u\| z_{2k-1}} (2 \pi \sqrt{-1} z_{2k-1})^{\mathbf j_1}}{-2 \sigma^2 z_{2k-1}} \big| \\
&\lesssim \sigma^{-2} \sum_{k = 1}^{\infty} e^{2 \pi \sqrt{-1} \|u\| z_{2k-1}} (z_{2k-1})^{\mathbf j_1 - 1},
\end{align*}
where $z_{2k-1} = a_k + \sqrt{-1} b_k$, and
\begin{align*}
a_k^2 + b_k^2 &= \sigma^{-2} (t_s^2 + ((2k-1) \pi)^2)^\frac{1}{2} \\
2 b_k^2 &= (a_k^2 + b_k^2) - (a_k^2 - b_k^2) = \sigma^{-2} \Big[ (t_s^2 + ((2k-1) \pi)^2)^\frac{1}{2} - t_s \Big]\\
&= \frac{\sigma^{-2} ((2k-1) \pi)^2}{(t_s^2 + ((2k-1) \pi)^2)^\frac{1}{2} + t_s}
\end{align*}

We begin with the first domain $D_1$, in which $\|s_{-1}\|^2 \leq (\sigma h)^{-2}$.
We need to set a threshold $H = (\frac{t_s}{\pi} + 1)/2$ for the index $k$.
When $k \leq \lf H \rf$, we have
\begin{align*}
2 b_k^2 \geq \sigma^{-2} \frac{(2k-1)^2 \pi^2}{3 t_s} &\Rightarrow b_k \gtrsim \sigma^{-1} \sqrt{t_s}^{-1} k \\
\sigma^{-2} t_s \leq a_k^2 + b_k^2 \leq \sqrt{2} \sigma^{-2} t_s &\Rightarrow |z_{2k-1}| = \Theta(\sigma^{-1} \sqrt{t_s});
\end{align*}
when $k \geq \lc H \rc$, we have
\begin{align*}
2 b_k^2 \geq \sigma^{-2} \frac{(2k-1)^2 \pi^2}{3 (2k-1) \pi} &\Rightarrow b_k \gtrsim \sigma^{-1} \sqrt{k} \\
\sigma^{-2} (2k-1) \pi \leq a_k^2 + b_k^2 \leq \sqrt{2} \sigma^{-2} (2k-1) \pi &\Rightarrow |z_{2k-1}| = \Theta(\sigma^{-1} \sqrt{k});
\end{align*}
and the series would be bounded by
\begin{align*}
\sigma^{-2} \sum_{k=1}^{\lf H \rf} e^{-2 \pi |u| b_k} |z_{2k-1}|^{\mathbf j_1 - 1} 
&\lesssim \sigma^{-2} (\sigma^{-1} \sqrt{t_s})^{\mathbf j_1 - 1} \sum_{k=1}^{\lf H \rf} {e^{-C_3 |u| \sigma^{-1} \sqrt{t_s}^{-1} k}} \\
&\lesssim \sigma^{-(\mathbf j_1 + 1)} \sqrt{t_s}^{\mathbf j_1 - 1} e^{-C_3 |u| \sigma^{-1} \sqrt{t_s}^{-1}} \\
\sigma^{-2} \sum_{k=\lc H \rc}^{\infty} e^{-2 \pi |u| b_k} |z_{2k-1}|^{\mathbf j_1 - 1} 
&\lesssim \sigma^{-2} \sum_{k=\lc H \rc}^{\infty} e^{-C_3 |u| \sigma^{-1} \sqrt{k}}(\sigma^{-1} \sqrt{k})^{\mathbf j_1 - 1} \\
&\lesssim \sigma^{-(\mathbf j_1 + 1)} (e^{-C_3 |u| \sigma^{-1} \sqrt{\lc H \rc}} \sqrt{\lc H \rc}^{\mathbf j_1 - 1} \\
&\quad + \int_{\lc H \rc}^{\infty} e^{-C_3 |u| \sigma^{-1} \sqrt{x}} \sqrt{x}^{\mathbf j_1 - 1} \dd x) \\
&\lesssim \sigma^{-(\mathbf j_1 + 1)} (e^{-C_3 |u| \sigma^{-1} \sqrt{\lc H \rc}} \sqrt{\lc H \rc}^{\mathbf j_1 - 1} \\
&\quad + 1/(C_3 |u| \sigma^{-1}) e^{-C_3 |u| \sigma^{-1} \sqrt{\lc H \rc}} \sqrt{\lc H \rc}^{\mathbf j_1}) \\
&\lesssim \sigma^{-(\mathbf j_1 + 1)} \sqrt{\lc H \rc}^{\mathbf j_1 - 1} e^{-C_3 |u| \sigma^{-1} \sqrt{\lc H \rc}}.
\end{align*}
We drop one term in the last line as $\sigma \sqrt{\lc H \rc} \leq \sigma h^{-1} \lesssim 1$, 
and note the first series would only appear when $t_s > \pi$, and $\lc H \rc = \Theta(\max(t_s/\pi, 1))$.

For the integral over the domain $D_1$, it would be bounded as
\begin{align*}
& \int_{D_1} \frac{e^{2\pi \sqrt{-1} \|u\| s_1}}{1 + \frac{\lambda}{\sigma^d}\,e^{\sigma^2 (\|s_{-1}\|^2 + s_1^2)}} \prod_{i=1}^d |s_i|^{\mathbf j_i} \dd s \\
\leq& \int_{\|s_{-1}\| \leq \frac{1}{\sigma h}} \prod_{i=2}^d |s_i|^{\mathbf j_i} \cdot \int_{-\infty}^\infty \frac{e^{2 \pi \sqrt{-1} \|u\| s_1} |s_1|^{\mathbf j_1}}{1 + \frac{\lambda}{\sigma^d}\,e^{\sigma^2 \|s_{-1}\|^2} e^{\sigma^2 s_1^2}} \dd s_1 \dd s_{-1} \\
\lesssim& \sigma^{-(\mathbf j_1 + 1)} \int_{\|s_{-1}\| \leq \frac{1}{\sigma h}} \|s_{-1}\|^{|\mathbf j| - {\mathbf j_1}} \cdot 
(\sqrt{t_s}^{\mathbf j_1 - 1} e^{-C_3 |u| \sigma^{-1} \sqrt{t_s}^{-1}} + \sqrt{\lc H \rc}^{\mathbf j_1 - 1} e^{-C_3 |u| \sigma^{-1} \sqrt{\lc H \rc}})\dd s_{-1} \\
\lesssim& \sigma^{-(\mathbf j_1 + 1)} \Big( \int_0^{\frac{1}{\sigma h}} r^{|\mathbf j| - {\mathbf j_1} + d-2} \sqrt{t_s}^{\mathbf j_1 - 1} e^{-C_3 |u| \sigma^{-1} \sqrt{t_s}^{-1}} \dd r \\
&\qquad \qquad + \int_0^{\frac{1}{\sigma h}} r^{|\mathbf j| - {\mathbf j_1} + d-2} \sqrt{\lc H \rc}^{\mathbf j_1 - 1} e^{-C_3 |u| \sigma^{-1} \sqrt{\lc H \rc}})\dd r \Big).
\end{align*}

To bound the first integral term, we should utilize a transformation $r = \frac{\sin(\theta)}{\sigma h}, t_s = h^{-1} \cos(\theta)$, and have
\begin{align*}
& \int_0^{\frac{1}{\sigma h}} r^{|\mathbf j| - {\mathbf j_1} + d-2} \sqrt{t_s}^{\mathbf j_1 - 1} e^{-C_3 |u| \sigma^{-1} \sqrt{t_s}^{-1}} \dd r \\
=& (\sigma h)^{-(|\mathbf j| - {\mathbf j_1} + d-2)} h^{-(\mathbf j_1 - 1)} \int_0^{\pi/2} \sin(\theta)^{|\mathbf j| - {\mathbf j_1} + d-2} \cos(\theta)^{\mathbf j_1 - 1} e^{-C_3 |u| \sigma^{-1} h / \cos(\theta)} \dd \frac{\sin(\theta)}{\sigma h} \\
\leq& (\sigma h)^{-(|\mathbf j| - {\mathbf j_1} + d-1)} h^{-(\mathbf j_1 - 1)} \int_0^{\pi/2} e^{-C_3 |u| \sigma^{-1} h} \dd \theta
\lesssim \sigma^{-(|\mathbf j| - {\mathbf j_1} + d-1)} h^{-(|\mathbf j| + d-2)} e^{-C_3 |u| \sigma^{-1} h};
\end{align*}
the second integral term could be addressed by utilizing the fact $\lc H \rc \geq 1$
\begin{align*}
& \int_0^{\frac{1}{\sigma h}} r^{|\mathbf j| - {\mathbf j_1} + d-2} \sqrt{\lc H \rc}^{\mathbf j_1 - 1} e^{-C_3 |u| \sigma^{-1} \sqrt{\lc H \rc}})\dd r \\
\leq& \int_{\sigma^{-1}(h^{-2} - \pi)^{\frac{1}{2}}}^{\frac{1}{\sigma h}} r^{|\mathbf j| - {\mathbf j_1} + d-2} e^{-C_3 |u| \sigma^{-1}} \dd r
+ \int_{0}^{\frac{1}{\sigma h}} r^{|\mathbf j| - {\mathbf j_1} + d-2} \sqrt{t_s}^{\mathbf j_1 - 1} e^{-C_3 |u| \sigma^{-1}} \dd r \\
\lesssim& (\sigma h)^{-(|\mathbf j| - {\mathbf j_1} + d-1)} h^{-(\mathbf j_1 - 1)} e^{-C_3 |u| \sigma^{-1}}
\int_0^{\pi/2} \sin(\theta)^{|\mathbf j| - {\mathbf j_1} + d-2} \cos(\theta)^{\mathbf j_1} \dd \theta \\
\leq& (\sigma h)^{-(|\mathbf j| - {\mathbf j_1} + d-1)} h^{-(\mathbf j_1 - 1)} e^{-C_3 |u| \sigma^{-1} h},
\end{align*}
which could infer the total integral over domain $D_1$ should be $\m O(\sigma^{-(|\mathbf j| + d)} h^{-(|\mathbf j| + d-2)} e^{-C_3 |u| \sigma^{-1} h})$.

Over the second domain $D_2$, we can similarly divide the series into two parts by the threshold $H$.
We notice $\sqrt{t_s}$ and $\sqrt{k}$ would respectively dominate the scale of $b_k$ or $|z_{2k-1}|$ in the two parts, as
\begin{align*}
2 b_k^2 &\geq \sigma^{-2} ((2k-1) \pi + t_s) \\
a_k^2 + b_k^2 &= \Theta(\sigma^{-2}((2k-1) \pi + t_s)).
\end{align*}
Using a similar derivation as above, the two parts would be correspondingly bounded by
\begin{align*}
\sigma^{-2} \sum_{k=1}^{\lf H \rf} e^{-2 \pi |u| b_k} |z_{2k-1}|^{\mathbf j_1 - 1} 
&\lesssim \sigma^{-(\mathbf j_1 + 1)} \sqrt{t_s}^{\mathbf j_1 - 1} \lf H \rf {e^{-C_3 |u| \sigma^{-1} \sqrt{t_s}}} \\
\sigma^{-2} \sum_{k=\lc H \rc}^{\infty} e^{-2 \pi |u| b_k} |z_{2k-1}|^{\mathbf j_1 - 1} 
&\lesssim \sigma^{-(\mathbf j_1 + 1)} \sqrt{\lc H \rc}^{\mathbf j_1 - 1} e^{-C_3 |u| \sigma^{-1} \sqrt{\lc H \rc}}.
\end{align*}
Since $\lc H \rc \gtrsim t_s$ and $\lc H \rc \geq 1$, the overall series could be bounded by a constant multiple of $\sigma^{-(\mathbf j_1 + 1)} \sqrt{\lc H \rc}^{\mathbf j_1 + 1} e^{-C_3 |u| \sigma^{-1} \sqrt{\lc H \rc}}$.
Therefore, by polar coordinate transformation, the integral over $D_2$ would be reduced to the following one, 
\begin{align*}
\int_{\frac{1}{\sigma h}}^\infty r^{|\mathbf j| - {\mathbf j_1} + d-2} \sqrt{\lc H \rc}^{\mathbf j_1 + 1} e^{-C_3 |u| \sigma^{-1} \sqrt{\lc H \rc}}) \dd r 
&= \int_{\frac{1}{\sigma h}}^{\frac{2}{\sigma h}} r^{|\mathbf j| - {\mathbf j_1} + d-2} \sqrt{\lc H \rc}^{\mathbf j_1 + 1} e^{-C_3 |u| \sigma^{-1} \sqrt{\lc H \rc}}) \dd r \\
+ &\int_{\frac{2}{\sigma h}}^\infty r^{|\mathbf j| - {\mathbf j_1} + d-2} \sqrt{\lc H \rc}^{\mathbf j_1 + 1} e^{-C_3 |u| \sigma^{-1} \sqrt{\lc H \rc}}) \dd r,
\end{align*}
where we further divide the integral based on whether $t_s > h^{-2}$.
For the first stage, utilizing $\lc H \rc \geq 1$, we can bound it as
\begin{align*}
\int_{\frac{1}{\sigma h}}^{\frac{2}{\sigma h}} r^{|\mathbf j| - {\mathbf j_1} + d-2} \sqrt{\lc H \rc}^{\mathbf j_1 + 1} e^{-C_3 |u| \sigma^{-1} \sqrt{\lc H \rc}}) \dd r
&\lesssim \int_{\frac{1}{\sigma h}}^{\frac{2}{\sigma h}} r^{|\mathbf j| - {\mathbf j_1} + d-2} h^{-(\mathbf j_1 + 1)} e^{-C_3 |u| \sigma^{-1}} \dd r \\
&\lesssim (\sigma h)^{-(|\mathbf j| - {\mathbf j_1} + d-1)} h^{-(\mathbf j_1 + 1)} e^{-C_3 |u| \sigma^{-1} h};
\end{align*}
for the second stage, we could apply $\lc H \rc = \Theta(\sigma^2 r^2)$, and have
\begin{align*}
\int_{\frac{2}{\sigma h}}^\infty r^{|\mathbf j| - {\mathbf j_1} + d-2} \sqrt{\lc H \rc}^{\mathbf j_1 + 1} e^{-C_3 |u| \sigma^{-1} \sqrt{\lc H \rc}}) \dd r
&\lesssim \int_{\frac{2}{\sigma h}}^\infty r^{|\mathbf j| - {\mathbf j_1} + d-2} (\sigma r)^{\mathbf j_1 + 1} e^{-C_3 |u| r} \dd r \\
&\lesssim \sigma^{\mathbf j_1 + 1} (\sigma h)^{-(|\mathbf j| + d-1)} e^{-C_3 |u| \sigma^{-1} h^{-1}},
\end{align*}
which implies the scale of the integral over $D_2$ is $\m O((\sigma h)^{-(|\mathbf j| + d)} e^{-C_3 |u| \sigma^{-1} h})$.
The second claim can thus be proved by simply combining the current results for $D_1$ and $D_2$.

\end{proof}

\section{NUMERICAL INTEGRATION}
\label{sec:polar}

To make the algorithm end in $\wt{\m O}(n)$ time, we need to efficiently compute all the leverage approximation (\ref{Eqn:lv_app}) in the main paper.
We first state an observation that the original multiple integral over $\mb R^d$ could be simplified to a normal integral with only one variable.
Then we propose a fast method to give the approximation of the integral, 
which only requires $\wt{\m O}(n)$ time to compute all the integrals.

\subsection{Simplify the Integration by Polar Coordinate Transformation}

An important feature of a \matern kernel is the isotropy that the value of the kernel function $K_\alpha(x)$ only depends on the module $\|x\|_2$ (for simplicity $\|\cdot\|_2$ would be denoted as $\|\cdot\|$ from then on in this section).
The property is shared by the corresponding spectral density ${m}_\alpha(s)$, and thus the Fourier transform of our rescaled leverage score approximation $\wt K_\lambda(\cdot, t)$ also inherits the isotropy.
In particular, given the center point $t$ and a point $x$ of interest, by Fourier transform formula,
\begin{align}
\label{eqn:mv_integral}
\wt K_\lambda(x, t) = \int_{\mb R^{d}} \ms F[\wt K_\lambda(\cdot, t)](s) \exp(2 \pi \sqrt{-1} x^T s) \dd s
= \int_{\mb R^{d}} \frac{\exp(2 \pi \sqrt{-1} (x-t)^T s)}{p(t) + \lambda / m_\alpha(s)} \dd s.
\end{align}
Considering the specific case $x = t$ in computing leverage score approximation, the value of the integrand would be the same for any $s$ with the same module $\|s\|$.
By polar coordinate transformation, we obtain
\begin{align*}
\wt K_\lambda(t, t) = \int_0^\infty \frac{1}{p(t) + \lambda / m_\alpha(r)} \cdot S_{d-1}(r) \dd r
\end{align*}
where $S_{d-1}(r)$ is the surface area of a $(d)$-dim ball with radius $r$.

It is worth mentioning for the general case $x \neq t$, the isotropy could also be utilized to accelerate the computation.
In some geostatistics literature, for example, Hankel transform \citep{kleiber2015spatial} is applied to simplify the integration into a univariate integral for two-dimensional processes with \matern kernels.
We extend this idea to kernels with dimension more than two and represent the rescaled leverage score approximation $\wt K_\lambda(x, t)$ as a double integral.
We notice the value of the integrand in equation~(\ref{eqn:mv_integral}) would be the same for any $s$ with the same module $\|s\|$ and the same inner product $(x-t)^T s$.
Since the spectral density $m_\alpha(s)$ only depends on $\|s\|$, we slightly abuse the notation, instead representing it as $m_\alpha(\|s\|)$ to emphasize the isotropy.
Moreover, by a certain coordinate transformation ${r = \|s\|, \cos(\theta) = \frac{(x-t)^T s}{\|x-t\| \|s\|} }$, we rewrite the integrand above as $ \exp(2 \pi \sqrt{-1} \|x-t\| r \cos(\theta))$, 
and observe that the integrand would remain unchanged with the input points from the intersection between the $(d-1)$-sphere $\{s \in \mb R^{d}: \|s\| = r\}$ and the cone $\{s \in \mb R^{d}: (x-t)^T s = \|x-t\| \|s\| \cos(\theta)\}$ (the intersection is indeed a $(d-2)$-sphere with radius $r \sin(\theta)$).
With those notations, the original $d$-dim integral would thereby be calculated as
\begin{align*}
\wt K_\lambda(x, t) = \int_0^\infty \int_0^\pi  \frac{\exp(2 \pi \sqrt{-1} \|x-t\| r \cos(\theta))}{p(t) + \lambda / m_\alpha(r)}  \cdot S_{d-2}(r \sin(\theta)) r \dd \theta \dd r
\end{align*}
where $S_{d-2}(r \sin(\theta))$ is the surface area of a $(d-1)$-dim ball with radius $r \sin(\theta)$. 
The same trick applies to all the other stationary kernels with isotropic spectral density function, including Gaussian kernels.

\subsection{Approximation of the Integrals}

Directly, the integrals above can be computed by a reliable package QUADPACK with the specific integrator QAWF~\citep{piessens2012quadpack}, which is targeted at Fourier cosine transform.
However, the computation is time-consuming, as QAWF implements an adaptive method so that when $\lambda \to 0$ it will require more function evaluations.
To overcome the potential drawback, we propose a fast method to approximate the integration with $o(1)$ relative error in near-constant time.

For \matern kernels, we would focus on the following integral of a simplified form,
\begin{align*}
\int_0^\infty \frac{x^{d-1}}{p + \lambda (1 + x^2)^\alpha} \dd x.
\end{align*}

Inspired by the derivation of the scale of the integral $O(\lambda^{-\frac{d}{2\alpha}})$, we rewrite the integral as
\begin{align*}
\int_0^\infty \frac{x^{d-1}}{p + (\lambda^{\frac{1}{\alpha}} + (\lambda^{\frac{1}{2\alpha}} x)^2)^\alpha} \dd x
= \lambda^{-\frac{d}{2\alpha}} \int_0^\infty \frac{x^{d-1}}{p + (\lambda^{\frac{1}{\alpha}} + x^2)^\alpha} \dd x,
\end{align*}
and intuitively want to replace $(\lambda^{\frac{1}{\alpha}} + x^2)$ with $x^2$.
We would show the approximation would only result in a small relative error of order $\m O(\lambda^{\frac{1}{\alpha}}) = o(1)$ as required.

The difference between the two integrands is
\begin{align*}
\lambda^{-\frac{d}{2 \alpha}} (\frac{x^{d-1}}{p + x^{2 \alpha}} - \frac{x^{d-1}}{p + (\lambda^{\frac{1}{\alpha}} + x^2)^\alpha}) 
= \lambda^{-\frac{1}{2 \alpha}} x^{d-1} \frac{(\lambda^{\frac{1}{\alpha}} + x^2)^\alpha - x^{2 \alpha}}{(p + x^{2 \alpha}) (p + (\lambda^{\frac{1}{\alpha}} + x^2)^\alpha)}.
\end{align*}
When $x^2 \leq \lambda^{\frac{1}{\alpha}}$, the numerator above would be bounded by $(2^\alpha - 1) \lambda$, and further we have
\begin{align*}
\frac{(\lambda^{\frac{1}{\alpha}} + x^2)^\alpha - x^{2 \alpha}}{(p + x^{2 \alpha}) (p + (\lambda^{\frac{1}{\alpha}} + x^2)^\alpha)}
\leq \frac{(2^\alpha - 1) \lambda}{(p + x^{2 \alpha}) (p + (\lambda^{\frac{1}{\alpha}} + x^2)^\alpha)} 
\lesssim \frac{\lambda^{\frac{1}{\alpha}}}{p + x^{2 \alpha}}.
\end{align*}
When $x^2 > \lambda^{\frac{1}{\alpha}}$, we can control the numerator by the first order Taylor approximation,
\begin{align*}
\frac{(\lambda^{\frac{1}{\alpha}} + x^2)^\alpha - x^{2 \alpha}}{(p + x^{2 \alpha}) (p + (\lambda^{\frac{1}{\alpha}} + x^2)^\alpha)}
\lesssim \frac{\lambda^{\frac{1}{\alpha}} (x^2)^{\alpha - 1}}{(p + x^{2 \alpha}) (p + (\lambda^{\frac{1}{\alpha}} + x^2)^\alpha)}
\lesssim \frac{\lambda^{\frac{1}{\alpha}}}{p + x^{2 \alpha}},
\end{align*}
where the last relation holds as $(x^2)^{\alpha - 1} \lesssim p + (\lambda^{\frac{1}{\alpha}} + x^2)^\alpha$.

Then the total difference between two integrals would be bounded by a constant multiple of
\begin{align*}
\lambda^{-\frac{d}{2 \alpha}} \int_0^\infty \frac{\lambda^{\frac{1}{\alpha}} x^{d-1}}{p + x^{2 \alpha}} \dd x,
\end{align*}
which is $O(\lambda^{-\frac{d}{2 \alpha}} \lambda^{\frac{1}{\alpha}})$.
Considering the magnitude of the original integral is $\Theta (\lambda^{-\frac{d}{2 \alpha}})$, our claim regarding the relative error is validated.

We utilize the formula $\int_0^\infty \frac{dx}{1+x^a} = \frac{\pi / a}{\sin{\pi / a}}$ to give the final approximation:
\begin{align*}
\int_0^\infty \frac{x^{d-1}}{p + \lambda (1 + x^2)^\alpha} \dd x \approx p^{\frac{d}{2 \alpha}-1} \frac{\lambda^{-\frac{d}{2 \alpha}}}{d} \frac{\pi / \frac{d}{2 \alpha}}{\sin{\pi / \frac{d}{2 \alpha}}}.
\end{align*}
As the sampling probability is computed as the normalized leverage, we can even directly use $p^{\frac{d}{2 \alpha}-1}$ as the rescaled leverage and ignore the rest factor.

For Gaussian kernels, the formula would be even easier, since there is a closed form expression for the target integral,
\begin{align*}
\frac{2}{\Gamma(d/2)}\int_0^\infty \frac{t^{d-1}}{p (2 \pi \sigma^2)^{d/2} + \lambda e^{t^2}} \dd t
= -\frac{Li_{d/2} (-\frac{p (2 \pi \sigma^2)^{d/2}}{\lambda})}{p (2 \pi \sigma^2)^{d/2}}
\end{align*}
where $\sigma$ is the bandwidth of the Gaussian kernel used,
and $Li_{d/2}(\cdot)$ is the polylogarithm function with order $\frac{d}{2}$.
The fast computation of the polylogarithm function has already been thoroughly studied by some previous works \citep{crandall2006note, vepvstas2008efficient, johansson2015rigorous}, 
and they proposed various methods to compute $Li_{d/2}(c)$ with $\Theta(\log \log n)$ bits of precision ($\Theta(\frac{1}{\log n})$ relative error) and a polynomial of $\log \log n$ time.
The total time to compute the leverage would thus still be $\wt{\m O}(n)$.

\section{DENSITY ESTIMATION} \label{sec:density_estimation}

As we described in the main paper, by utilizing the distributional information, leverage scores in a KRR problem could be efficiently approximated by our analytical method. 
In the case we do not have prior knowledge of the distribution, we propose to estimate densities of data points via kernel density estimation, 
which have been discussed in the main paper that by using some recent KDE methods with a sub-optimal error rate, we can perform the density estimation within $\wt {\m O}(n)$ time.
To further justify the usage of kernel density estimation, we imply by the following lemma that given an $o(1)$ error in KDE, 
the particular error in approximating statistical leverage scores due to the density estimation is asymptotically negligible.
\begin{lemma}
	Under the same assumptions before, $\forall x \in spt(p)$ (the support of $p$), given a fixed point $t$, 
	the supremum of the error caused by the density estimate $\hat{p}(t)$ on point $t$, 
	is bounded by a constant multiple of $h^{-d} |p(t) - \wht{p}(t)|$. 
\end{lemma}
\begin{proof}
For simplicity, we first denote the leverage score approximation with estimated density as $\wht{\widetilde K}_\lambda(x, x_i)$. 
Inserting the density estimation $\hat p$ into the formula of rescaled leverage scores (equation~\ref{eqn:eqv_kernel} in the main paper), we obtain $\ms F[\widehat{\widetilde K}_\lambda(\cdot, x_i)](s) = \frac{e^{-2 \pi \sqrt{-1} \dotp{x_i}{s}}}{\wht p(x_i) + \lambda  (m_\alpha(s))^{-1}}$.

By triangle inequality, the supremum of the total error $|\widehat{\widetilde K}_\lambda(x, x_i) - G(x, x_i)|$ could be divided into two sources, 
one due to density estimation $|\widehat{\widetilde K}_\lambda(x, x_i) - {\widetilde K}_\lambda(x, x_i)|$ and the other due to approximation error $|{\widetilde K}_\lambda(x, x_i) - G_\lambda(x, x_i)|$, which has been thoroughly discussed in the appendix. 
Here we focus on the first term.
\begin{align*}
& \sup_{x \in \mathbb{R}^d} |\wht{\widetilde K}_\lambda(x, x_i) - {\widetilde K}_\lambda(x, x_i)| \\
=& \sup_{x \in \mathbb{R}^d} \Big|\ms F^{-1} \big[\ms F[\wht{\widetilde K}_\lambda(\cdot, x_i)]\big](x) - \ms F^{-1} \big[\ms F[{\widetilde K}_\lambda(\cdot, x_i)]\big](x) \Big| \\
=& \sup_{x \in \mathbb{R}^d} \Big| \ms F^{-1} \big[\frac{|\wht{p}(x_i) - p(x_i)| e^{-2 \pi \sqrt{-1} \dotp{x_i}{s}}}
{(\wht p(x_i) + \lambda \cdot (m_\alpha(s))^{-1}) (p(x_i) + \lambda \cdot (m_\alpha(s))^{-1})}\big](x) \Big|
\end{align*}
$|\wht{p}(x_i) - p(x_i)|$ could be extracted as a factor, and we solely need to deal with the rest term
\begin{align*}
&\sup_{x \in \mathbb{R}^d} \cdot \Big|\ms F^{-1} \big[\frac{e^{-2 \pi \sqrt{-1} \dotp{x_i}{s}}}
{(\wht p(x_i) + \lambda \cdot (m_\alpha(s))^{-1}) (p(x_i) + \lambda \cdot (m_\alpha(s))^{-1})}\big](x) \Big| \\
= &\sup_{x \in \mathbb{R}^d} \Big|\int_{\mb R^d} 
\big[\frac{e^{-2 \pi \sqrt{-1} \dotp{x_i-x}{s}}} {(\wht p(x_i) + \lambda \cdot (m_\alpha(s))^{-1}) (p(x_i) + \lambda \cdot (m_\alpha(s))^{-1})}\big] ds \Big|
\end{align*}
Relaxing the exponential term as $1$ and using Cauchy-Schwarz inequality, we obtain
\begin{align*}
&\Big|\int_{\mb R^d} \big[\frac{e^{-2 \pi \sqrt{-1} \dotp{x_i-x}{s}}} {(\wht p(x_i) + \lambda \cdot (m_\alpha(s))^{-1}) (p(x_i) + \lambda \cdot (m_\alpha(s))^{-1})}\big] ds \Big| \\
\leq &\Big|\int_{\mb R^d} \big[\frac{1} {(\wht p(x_i) + \lambda \cdot (m_\alpha(s))^{-1}) (p(x_i) + \lambda \cdot (m_\alpha(s))^{-1})}\big] ds \Big| \\
\leq &\Big\|\frac{1} {\wht p(x_i) + \lambda \cdot (m_\alpha(s))^{-1}} \Big\|_2 \cdot
\Big\|\frac{1} {p(x_i) + \lambda \cdot (m_\alpha(s))^{-1}} \Big\|_2 \\
\lesssim &h^{-d/2} h^{-d/2} = h^{-d}
\end{align*}
The last inequality can be verified by Lemma~\ref{Lem:EqKernel_property_matern}.
\end{proof}

We finally remark that given the $o(1)$ factor $|\wht{p}(x_i) - p(x_i)|$, 
the error $|\widehat{\widetilde K}_\lambda(x, x_i) - {\widetilde K}_\lambda(x, x_i)|$ caused by the density estimation would therefore be $o(h^{-d})$,
and thus is enough to make the relative error of leverage approximation vanish.

\subsection{Missing assumptions for modified HBE}
\label{Sec:hbe}

We list some advanced KDE methods in the main paper to show theoretically we can estimate the density with time complexity at most polynomial in the dimension $d$. 
Among them, modified Hashing-Based Estimators (HBE) \citep{backurs2019space} is the most recent one. 
Taking this method as a representative, we copy the assumption in modified HBE here for the sake of completeness.
\begin{assumption}[$(\frac12 ,M)$-LSHable]
Let $\m K_e(x,y)$ be the kernel function used for KDE, for which there exists a distribution $H$ of hash functions and $M \geq 1$ such that for every $x,y \in \mb R^d$,
\begin{align*}
M^{-1} \cdot \m K_e(x,y)^{\frac12} \leq {\mb P}_{h \sim H} \left\{ h(x) = h(y) \right\} \leq M \cdot \m K_e(x,y)^{\frac12}.
\end{align*}
\end{assumption}
To attain the fast rate claimed by modified HBE, the core assumption above that the kernel used for KDE is $(\frac12, M)$-LSHable for some constant $M$ is necessary.
The authors have proved that some common kernels, such as Laplacian and exponential kernels, are $(\frac12, \m O(1))$-LSHable;
and thus a density estimator based on those kernels can be efficiently approximated by modified HBE.

\section{TECHNICAL RESULTS}
\label{Sec:int_by_parts}

Some tricks in multivariate integrals are heavily utilized in this work, and here we present a lemma to address the technical details about it.
We first would like to mention the notation $\int_{\mb R^d} f(x) \dd F_n(x)$ in our paper is not strict in general, as the multivariate version Riemann–Stieltjes integral is not well defined. 
In this appendix we just abuse the integral $\int_{\mb R^d} f(x) \dd F_n(x)$ to represent the summation $\frac{1}{n} \sum_{i=1}^n f(x_i)$, and the integral $\int_{\mb R^d} f(x) \dd F(x)$ to represent the expectation $\int_{\mb R^d} f(x) p(x) \dd x$.

The lemma is presented as follows.
(cf. Section~\ref{Sec:embedding} for the notations in the lemma.)
\begin{lemma}[Multivariate integration by parts]
\label{thm:int_by_parts}
Given the absolute continuous approximation $F$ with the compact support $\Omega$ and $L_\infty$ density $p$, 
the empirical distribution $F_n$, and an integrand $g(\cdot) \in W^{\alpha, 2}$ independent of $F_n$,
the certain integral of interest $\int_{C(y, \delta)} g(x) \dd (F_n - F)(x)$ is almost surely (considering the samples in $F_n$ are drawn from $F$) equal to
\begin{align*}
\sum_{{\mathbf A \sqcup \mathbf B \sqcup \mathbf C = [d]}} (-1)^{|\mathbf A| + |\mathbf B|} 
&\int_{C(y_{\mathbf A}, \delta)} D^{\mathbf I_{\mathbf A}} g \Big(x_{\mathbf A}; \big(y + \delta \cdot (-I_{\mathbf B} + I_{\mathbf C}) \big)_{\mathbf B \sqcup \mathbf C} \Big) \cdot \\
&\quad\quad\quad \Big( F_n - F \Big)\big(x_{\mathbf A}; \big(y + \delta \cdot (-I_{\mathbf B} + I_{\mathbf C}) \big)_{\mathbf B \sqcup \mathbf C}\big)
\dd x_{\mathbf A},
\end{align*}
where $\sqcup$ is the notation for disjoint union.
The sets $\mathbf B$ and $\mathbf C$ indeed indicate the certain dimensions to which lower and upper limits are assigned separately.
Specifically, if $C(x_0, \delta) = \mb R^d (\delta = \infty)$ and $g(x)$ vanishes at infinity, %
$\int_{\mb R^d} g(x) \dd (F_n - F)(x) = (-1)^d \int_{\mb R^d} (F_n - F)(x) \frac{\partial^d g(x)}{\partial x_1 \partial x_2 \cdots \partial x_d} \dd x$;
if $g(x)$ and its mixed derivative (up to order $\alpha$) vanish at infinity and are $L-$Lipschitz, the claim would hold without the assumption on the independence between $g$ and $F_n$.
\end{lemma}

\begin{proof}
Without loss of generality, we would illustrate our claim by a special $2$-d case to avoid the tedious calculation.
We would first prove the conclusion for an indefinitely differentiable density $p_{n, \epsilon}(x) = \frac{1}{n} \sum_{i=1}^n \eta_\epsilon(x-x_i)$, where the heat kernel $\eta_\epsilon(x) \defeq \frac{1}{\sqrt{2 \pi \epsilon}} \exp(-\frac{\dotp{x}{x}}{2 \epsilon})$ of the Dirac delta function.
(From then on in this proof, $x_i$ means the $i$-th element of the vector $x$.)
As a sketch of the proof, we would first show the lemma holds for the integral $\int_{C(y, \delta)} g(x) \big(p_{n, \epsilon}(x)-p(x)\big) \dd x$, and finally prove as $\epsilon \to 0$, the integral would converge to the claimed expression in this lemma.

For simplicity, we denote $q(x) = p_{n, \epsilon}(x) - p(x)$ and $Q(x)$ is the corresponding distribution function.
We further denote $Q_1(x_1; x_2) \defeq \int_{-\infty}^{x_1} q(t, x_2) \dd t$ as a 1-d distribution function with a parameter $x_2$, so that Riemann–Stieltjes integral is applicable to $Q_1(x_1; x_2)$.
By definition $\int_{-\infty}^{x_2} Q_1(x_1; t) \dd t = Q(x_1, x_2)$, and with that we have
\begin{align*}
\int_{C(y, \delta)} g(x) q(x) \dd x &= \int_{y_2 - \delta}^{y_2+\delta} \int_{y_1 - \delta}^{y_1+\delta} g(x_1, x_2) q(x_1, x_2) \dd x_1 \dd x_2 \\
&= \int_{y_2 - \delta}^{y_2+\delta} \int_{y_1 - \delta}^{y_1+\delta} g(x_1, x_2) \dd Q_1(x_1; x_2) \dd x_2
\end{align*}
We can safely apply integration by parts to the inside integral and obtain:
\begin{align} 
&\int_{C(y, \delta)} g(x) q(x) \dd x = \int_{y_2 - \delta}^{y_2+\delta} \Big( g(x_1, x_2) Q_1(x_1; x_2) \big|_{y_1 - \delta}^{y_1+\delta} - \int_{y_1 - \delta}^{y_1+\delta} Q_1 \frac{\partial g}{\partial x_1} \dd x_1 \Big) \dd x_2 \\
=& \int_{y_2 - \delta}^{y_2+\delta} g(y_1 + \delta, x_2)Q_1(y_1 + \delta, x_2) - g(y_1 - \delta, x_2)Q_1(y_1 - \delta, x_2) \dd x_2
- \int_{C(y, \delta)} Q_1 \frac{\partial g}{\partial x_1} \dd x.
\label{eqn:multi_int_by_parts}
\end{align}
Now we expand the original integral into three terms.
By repeatedly applying integration by parts to the first two terms, we have:
\begin{align*}
& \int_{y_2 - \delta}^{y_2+\delta} g(y_1 + \delta, x_2)Q_1(y_1 + \delta, x_2) \dd x_2 = 
g(y_1 + \delta, y_2 + \delta) Q(y_1 + \delta, y_2 + \delta) \\
&\quad\quad\quad - g(y_1 + \delta, y_2 - \delta) Q(y_1 + \delta, y_2 - \delta) - 
\int_{y_2 - \delta}^{y_2+\delta} Q_1(y_1 + \delta, x_2) \frac{\partial g(y_1 + \delta, x_2)}{\partial x_2} \dd x_2 \\
& \int_{y_2 - \delta}^{y_2+\delta} - g(y_1 - \delta, x_2)Q_1(y_1 - \delta, x_2) \dd x_2 = 
-g(y_1 - \delta, y_2 + \delta) Q(y_1 - \delta, y_2 + \delta) \\
&\quad\quad\quad + g(y_1 - \delta, y_2 - \delta) Q(y_1 - \delta, y_2 - \delta) + 
\int_{y_2 - \delta}^{y_2+\delta} Q_1(y_1 + \delta, x_2) \frac{\partial g(y_1 - \delta, x_2)}{\partial x_2} \dd x_2
\end{align*}
For the last term, we need to change the order of integration and have,
\begin{align*}
& - \int_{C(y, \delta)} Q_1 \frac{\partial g}{\partial x_1} \dd x \\
=& - \int_{y_1 - \delta}^{y_1+\delta} \int_{y_2 - \delta}^{y_2+\delta} \frac{\partial g(x_1, x_2)}{\partial x_1} \dd Q(x_1, x_2) \dd x_1 = - \int_{y_1 - \delta}^{y_1+\delta} Q(x_1, y_2+\delta) \frac{\partial g(x_1, y_2+\delta)}{\partial x_1} \dd x_1 \\
&\quad\quad\quad + \int_{y_1 - \delta}^{y_1+\delta} Q(x_1, y_2-\delta) \frac{\partial g(x_1, y_2-\delta)}{\partial x_1} \dd x_1
+ \int_{C(y, \delta)} Q \frac{\partial^2 g}{\partial x_1 \partial x_2} \dd x
\end{align*}
Summing up all the nine terms above, we would exactly obtain the claimed equation in the lemma.
In particular, if $C(y, \delta) = \mb R^d (\delta = \infty)$ and $g(x)$ vanishes at infinity, the first two terms in equation~(\ref{eqn:multi_int_by_parts}) would be dropped, 
and finally the only term left is $\int_{\mb R^d} g(x) q(x) \dd x$, which is equal to $(-1)^d \int_{\mb R^d} Q(x) \frac{\partial^d g(x)}{\partial x_1 \partial x_2 \cdots \partial x_d} \dd x$ as claimed.

To complete the proof, we still need to show the convergence.
We would begin with the assumption $g(\cdot)$ belongs to a dense subset $\ms D \subset W^{\alpha, 2}$, where $\ms D$ is the space of test functions.
Here we simply borrow some definitions and notations from the book~\citep[Chapter 6]{debnath2005introduction}, with respect to test functions and weak distributional convergence.
A test function is defined as an infinitely differentiable function on $\mb R^d$ vanishing outside of some bounded set.
We denote weak distributional convergence for a sequence of distributions $(P_m)$ to $P$ as $P_m \to P$ if $\dotp{P_m}{g} \to \dotp{P}{g}, \forall g \in \ms D$.
We can see our choice $P_{n, \epsilon} \to F_n$ in the weak distributional sense by setting $P_m = P_{n, 1/m}$,
and thus for the left hand side of our claim, $\int_{C(y, \delta)} g(x) \dd (P_{n, 1/m} - F)(x) \to \int_{C(y, \delta)} g(x) \dd (F_n - F)(x)$ by some standard techniques;
For the terms in the right hand side, we would illustrate by taking $\int_{C(y, \delta)} D^{\mathbf I_{[d]}} g(x) (P_{n,\epsilon} - P)(x) \dd x$ as an example.
We note $D^{\mathbf I_{\mathbf A}} g$ is still a test function, and $P_{n,\epsilon}$ converges to $F_n$ in $L_2$, so that by Cauchy-Shwartz inequality we could obtain,
\begin{align*}
&\big| \int_{C(y, \delta)} D^{\mathbf I_{[d]}} g(x) (P_{n,\epsilon} - F_n)(x) \dd x \big| \\
\leq& 
\big( \int_{C(y, \delta)} | D^{\mathbf I_{[d]}} g(x) |^2 \dd x \cdot \int_{C(y, \delta)} | (P_{n,\epsilon} - F_n)(x) |^2 \dd x \big)^{\frac{1}{2}} \to 0,
\end{align*}
which implies our desired convergence.

The next step is to extend test functions to functions in $W^{\alpha, 2}$.
An important fact is that $\ms D$ is a dense subspace of $W^{\alpha, 2}$, 
and we can find a sequence of test functions converging to $g$ in $W^{\alpha, 2}$ and further a subsequence $(g_m)$ converging to $g$ both in $W^{\alpha, 2}$ and almost everywhere, 
since convergence in $W^{\alpha, 2}$ implies convergence in $L_2$.
The Cauchy-Shwartz inequality trick for the right hand side would still work as this time the $L_2$ norm of $D^{\mathbf I_{[d]}} \big(g(x) - g_m(x)\big)$ goes to zero.
For the left-hand side, 
as we assume the $n$ samples are drawn from the absolutely continuous distribution $F$, 
for the certain sequence above we can show
\begin{align*}
|\int_{C(y, \delta)} \big(g(x) - g_m(x)\big) \dd F(x)| &\leq \int_{C(y, \delta)} |g(x) - g_m(x)| p(x) \dd x \\
&\leq \|g(x) - g_m(x)\|_{1, \Omega} \cdot \|p\|_\infty \\
&\leq |\Omega|^{\frac{1}{2}} \|p\|_\infty \|g(x) - g_m(x)\|_2 \to 0,
\end{align*}
and with probability one, over the $n$ sample points $g_m$ would pointwisely go to $g$.
We can thus conclude our claim would hold for functions in $W^{\alpha, 2}$ almost surely.

Finally, for the special case in which $g(x)$ and its mixed derivative (up to order $\alpha$) are $L-$Lipschitz, we are able to construct a convergent test function sequence $g_m$ which would pointwisely converge to $g$ over the compact support $\Omega$ containing all the samples in $F_n$.
We first introduce a sequence of test functions
\begin{align*}
\phi_m \defeq 
\left\{
     \begin{array}{lr}
         \frac{m^{\frac{d(d+1)}{2}}}{\varphi} e^{(\|m^{\frac{d+1}{2}} x\|^2 - 1)^{-1}}, &\qquad \text{if~} \|x\| < m^{-\frac{d+1}{2}}, \\
         0, &\qquad \text{otherwise},
     \end{array}
\right.
\end{align*}
where $\varphi = \int \phi_1(x) \dd x$ is the normalization factor.
The sequence $\{g_m\}$ is constructed as $\{\phi_m * (g \cdot 1_{m \Omega})\}$, 
the convolution of the test function $\phi_m$ and the truncation $g \cdot 1_{m \Omega}$, 
which is still a sequence of test functions.
It can be shown the sequence $g_m$ would go to $g$ in $W^{\alpha, 2}$.
To validate it, we need to observe the fact that for $x \in m\Omega$, 
\begin{align*}
|g_m(x) - g(x)| &= |\int \phi_m(t) \big( g(x-t) - g(x) \big) \dd t| 
\leq \int_{t<m^{-\frac{d+1}{2}}} \phi_m(t) | g(x-t) - g(x) | \dd t \\
&\leq \frac{L}{m^{\frac{d+1}{2}}} \int_{t<m^{-\frac{d+1}{2}}} \phi_m(t) \dd t = \frac{L}{m^{\frac{d+1}{2}}},
\end{align*}
where the second inequality holds because of the Lipschitz continuity.
Therefore, 
\begin{align*}
\int |g_m(x) - g(x)|^2 \dd x &= \int_{m \Omega} |g_m(x) - g(x)|^2 \dd x + \int_{\mb R^d - m \Omega} |g_m(x) - g(x)|^2 \dd x \\
&\leq \frac{L^2}{m^{d+1}} |m \Omega| 
+ 2 \int_{\mb R^d - m \Omega} g^2(x) \dd x + 2 \int_{\mb R^d - m \Omega} g_m^2(x) \dd x.
\end{align*}
Note the first term is proportional to $1/m$, and $g(x)$ vanishes at infinity.
The first two terms would both go to zero as $m \to \infty$.
For the last term, we could utilize Jensen's inequality and have $g_m^2(x) \leq \int \phi_m(t) g^2(x-t) 1_{m \Omega}(x-t) \dd t \leq \int \phi_m(t) g^2(x-t) \dd t$.
Then,
\begin{align*}
\int_{\mb R^d - m \Omega} g_m^2(x) \dd x &\leq \int_{\mb R^d - m \Omega} \int_{\mb R^d} \phi_m(t) g^2(x-t) \dd t \dd x 
= \int_{\mb R^d} \phi_m(t) \int_{\mb R^d - m \Omega} g^2(x-t) \dd x \dd t \\
&\leq \int_{\mb R^d} \phi_m(t) \int_{(\mb R^d - m \Omega) + B(m^{-\frac{d+1}{2}})} g^2(x) \dd x \dd t \\ 
&= \int_{(\mb R^d - m \Omega) + B(m^{-\frac{d+1}{2}})} g^2(x) \dd x \to 0
\end{align*}
where $B(m^{-\frac{d+1}{2}})$ is a ball with radius $m^{-\frac{d+1}{2}}$,
and the last convergence holds again since $g(x)$ vanishes at infinity.
Combining the pieces above, we can see $g_m \to g$ in $L_2$ and the similar conclusion holds for all its mixed derivatives up to order $\alpha$, which means $g_m \to g$ in $W^{\alpha, 2}$.
In that case, the convergence for the right-hand side of our claim would still hold as in the paragraph above.
For the left hand side, this time $g_m$ would uniformly converge to $g$ over $\Omega$, 
and we can show the integral $\int_{C(y, \delta)} g_m(x) \dd (F_n - F)(x)$ would converge to $\int_{C(y, \delta)} g(x) \dd (F_n - F)(x)$, even if $g$ depends on the empirical distribution $F_n$.

\end{proof}


\end{document}